\documentclass[authoryear,12pt,times]{elsarticle}
\usepackage[utf8]{inputenc}

\usepackage{latexsym,amsmath,amscd,amssymb,graphics}
\usepackage{enumerate}
\usepackage{wrapfig}
\usepackage{graphicx}
\usepackage{framed}
\usepackage{color}
\usepackage{float}
\usepackage{url}

\usepackage[all]{xy}

\makeatletter

\usepackage[all]{xy}

\newtheorem{theorem}{Theorem}

\newtheorem{lemma}[theorem]{Lemma}

\newtheorem{proposition}[theorem]{Proposition}

\newtheorem{remark}[theorem]{Remark}

\numberwithin{theorem}{section}
\newenvironment{proof}[1][Proof]{\textbf{#1.} }{\ \rule{0.5em}{0.5em}}

\def\be{\begin{equation}}
\def\ee{\end{equation}}
\def\bea{\begin{eqnarray}}
\def\eea{\end{eqnarray}}
\def\ba{\begin{array}}
\def\ea{\end{array}}

\def\bp{\mathbf{p}}
\def\bq{\mathbf{q}}

\def\bx{\mathbf{x}}

\newcommand{\rem}[1]{}

\newcommand{\de}{\delta}

\newcommand{\bu}{\boldsymbol{u}}

\newcommand{\bv}{\boldsymbol{v}}

\newcommand{\balpha}{\boldsymbol{\alpha}}
\newcommand{\bbeta}{\boldsymbol{\beta}}
\newcommand{\bxi}{\boldsymbol{\xi}}

\newcommand{\bmu}{\boldsymbol{\mu}}

\newcommand{\bnu}{\boldsymbol{\nu}}

\newcommand{\pp}[2]{\frac{\partial #1}{\partial #2}}

\newcommand{\mse}{\mathfrak{se}}
\newcommand{\mso}{\mathfrak{so}}

\newcommand{\todo}[1]{\vspace{5 mm}\par \noindent
\framebox{\begin{minipage}[c]{0.95 \textwidth}
\tt #1 \end{minipage}}\vspace{5 mm}\par}
\begin{document}
\begin{frontmatter}

\title{CLPNets: Coupled  Lie-Poisson Neural Networks for Multi-Part Hamiltonian Systems with Symmetries}

\author[1]{Christopher Eldred}
\ead{celdred@sandia.gov}
\author[2]{Fran\c{c}ois Gay-Balmaz}
\ead{francois.gb@ntu.edu.sg}
\author[3]{Vakhtang Putkaradze\corref{cor1}}
\cortext[cor1]{Corresponding author}
\ead{putkarad@ualberta.ca}

\affiliation[1]{organization = {Computer Science Research Institute}, 
addressline={Sandia National Laboratory, 1450 Innovation Pkwy SE}, city={Albuquerque}, state = {NM}, postcode={87123}, country={USA}}

\affiliation[2]{organization = {Division of Mathematical Sciences}, 
addressline={Nanyang Technological University}, city={637371}, country={Singapore}}

\affiliation[3]{organization={Department of Mathematical and Statistical Sciences},
            addressline={University of Alberta}, 
            city={Edmonton},
            postcode={T6G 2G1}, 
            state={Alberta},
            country={Canada}}
\date{\today}
\begin{abstract} 
To accurately compute data-based prediction of Hamiltonian systems, especially the long-term evolution of such systems, it is essential to utilize methods that preserve the structure of the equations over time. In this paper, we consider a case that is particularly challenging for data-based methods: systems with interacting parts that do not reduce to pure momentum evolution. Such systems are essential in scientific computations. For example, any discretization of a continuum elastic rod can be viewed as interacting elements that can move and rotate in space, with each discrete element moving on the configuration manifold, which can be viewed as the group of rotations and translations $SE(3)$. 
For an isolated elastic rod, the system is invariant with respect to the same group of rotations and translations. For individual elastic elements, due to this symmetry, the equations can be written as Lie-Poisson equations for the momenta, which are vector variables. For several elastic elements, the evolution involves not only the momenta but also the relative positions and orientations of the particles. Although such a system can technically be written in a way similar to a Lie-Poisson system with a Poisson bracket, the presence of Lie group-valued elements, such as relative positions and orientations, poses a problem for applying previously derived methods for data-based computing.

In this paper, we develop a novel method of data-based computation and complete phase space learning of such systems. We follow the original framework of \emph{SympNets} \citep{jin2020sympnets}, building the neural network from phase space mappings that are canonical, and transformations that preserve the Lie-Poisson structure (\emph{LPNets}) as in \citep{eldred2024lie}. To find adequate neural networks for Poisson systems including relative orientations, we derive a novel system of mappings that are built into neural networks describing the evolution of such systems. We call such networks Coupled Lie-Poisson Neural Networks, or \emph{CLPNets}. We consider increasingly complex examples for the applications of CLPNets, starting with the rotation of two rigid bodies about a common axis, progressing to the free rotation of two rigid bodies, and finally to the evolution of two connected and interacting $SE(3)$ components, describing the discretization of an elastic rod into two elements. Our method preserves all Casimir invariants of each system to machine precision, irrespective of the quality of the training data, and preserves energy to high accuracy. Our method also shows good resistance to the curse of dimensionality, requiring only a few thousand data points for all cases studied, with the effective dimension varying from three in the simplest cases to eighteen in the most complex. Additionally, the method is highly economical in memory requirements, requiring only about 200 parameters for the most complex case considered. 

\end{abstract}

\end{frontmatter}

\section{Introduction}
\subsection{Physical relevance and importance for applications}

Many physical systems contain several interacting parts of a similar nature. Individual parts of these systems possess their own dynamics and simultaneously contribute to the combined dynamics of the system through their interaction with other parts. Perhaps the most apparent examples of such systems are provided by the discretization of continuum mechanics, for example, elastic rods in the Lagrangian setting. A discretization based on either finite elements or a lattice will lead to a finite number of interacting elastic elements. These elements have their own linear and angular momenta, and the elastic energy depends on the \emph{relative} position and orientation of these elements, with interactions limited to the nearest neighbors. Thus, understanding how to solve the combined dynamics of coupled systems is essential for practical computations in continuum mechanics.
Besides the discretization of continuum elastic systems, other examples include systems modeling macromolecular conformations, such as those described by \citep{banavar2007structural}, which use finite-size elements. Applications of this type abound in robotics, especially in soft robots and manipulators \citep{chirikjian1994hyper,mengaldo2022concise}. The goal of this paper is to develop data-based, structure-preserving computing methods for such systems. Other physical examples of these kinds of systems include the dynamics of tree-like structures   \citep{amirouche1986dynamic,ider1989influence}, see also \citep{nieuwenhuis2018dynamics} for applications of tree-like structures to energy harvesting. Another application is the dynamics of polymers that may have nonlocal interactions through, for example, electrostatic charges  \citep{ellis2010symmetry}, and may have non-trivial dendronized topology in their structure  \citep{gay2012exact}.

In order to compute the evolution of these systems in a traditional way, one needs to assume an exact formula for the Hamiltonian of the system. In some cases, such a formula may be highly complex, involving many parameters and even having an unclear functional form.
That ambiguity is especially pronounced in the elasticity arguments related to biological materials, as numerous suggestions for elastic energy have been proposed. This is evident, for instance, in the elastic energy components of brain tissue mechanics \citep{jiang2015measuring,de2016constitutive,mihai2017family}. In addition to the inherent complexity of elasticity, the dynamics of brain tissue are believed to include viscoelastic components \citep{calhoun2019beyond}. Thus, for cases where the exact form of mathematical models is unknown, it is advantageous to develop data-based modeling methods. These methods circumvent the complexity of constructing exact models for simulations and only require the general form of these mathematical models, such as the variables they may depend on or the symmetries they may obey. 

We thus develop data-based numerical methods that fit into the framework of Physics-Informed Neural Networks (PINNs) for these systems. In this paper, we focus on developing computational methods applicable to Hamiltonian systems where friction can be neglected. As we describe below, traditional PINNs are difficult to apply directly to Hamiltonian systems, so we seek structure-preserving methods capable of accurately describing the systems' long-term dynamics. We begin with a review of the current literature on data-based computation of Hamiltonian systems.

\subsection{Relevant work related to data-based computing of Poisson systems}

Purely data-based methods rooted on Machine Learning (ML) have been successful in interpreting large amounts of unstructured data. However, applying ML methods directly to big data for scientific and industrial purposes, without a deep understanding of the underlying engineering and physics, presents challenges. To tackle these issues,  \textit{Physics Informed Neural Networks (PINNs)} have been developed \citep{raissi2019physics}. Given the breadth of PINNs literature, a comprehensive overview is beyond the scope here; interested readers are referred to recent reviews \citep{karniadakis2021physics,cuomo2022scientific,kim2024review} for references, literature reviews, and discussions on the advantages and disadvantages of these methods. In this approach, the system's evolution over time is assumed to be described by a quantity $ \mathbf{u} (t)$ belonging to some phase space, and the system's motion $ \mathbf{u} (t)$ is approximated by a governing equation: 
\begin{equation} 
\dot{\mathbf{u} } =\mathbf{f} ( \mathbf{u} ,t)\,.
\label{general_eq}
\end{equation} 
Here, $ \mathbf{u} $ and $\mathbf{f} $ can be either finite-dimensional, forming a system of ODEs, or infinite-dimensional, forming a system of PDEs where they depend, for example, on additional spatial variables $\bx$ and derivatives with respect to these variables. In the regular PINNs approach, the initial conditions, boundary conditions, and the ODEs themselves form part of the loss function. Subsequently, a neural network is designed to take the independent variable time $t$ (and perhaps the independent spatial variables $\bx$) and predict the unknown function $ \mathbf{u} (t)$ representing the solution. Due to the structure of neural networks, for a known set of weights, the derivatives of the solution with respect to the independent variables, and thus the function representing the equation \eqref{general_eq}, $\dot{ \mathbf{u} } - \mathbf{f}( \mathbf{u},t)$, can be computed analytically at given points, as well as the value of the function $ \mathbf{u} (t)$ at specified points (for example, initial conditions) and, if necessary, boundary conditions. For a given set of weights and activating functions in a neural network, the \emph{exact} values of temporal (and, if necessary, spatial) derivatives are computed using the method of automatic differentiation \citep{baydin2018automatic} to machine precision. Consequently, the loss can be computed as a function of the neural network's weights and optimized during the learning procedure. PINNs offer computational efficiency, with reported speedups exceeding 10,000 times 
for evaluations of solutions in complex systems like weather \citep{bi2023accurate,pathak2022fourcastnet}, accompanied by a drastic reduction of the energy consumption needed for computation.   PINNs are thus highly valuable in practical applications and have been implemented as open-source software in \emph{Python} and \emph{Julia}. Readers can easily find numerous open-source packages for implementing PINNs on GitHub, developed as individual or group projects, as well as products from large industry efforts, such as Nvidia's \emph{Modulus}.

In spite of their successes, there are still many uncertainties in the applications of PINNs. For example, PINNs may struggle with systems that have dynamics with widely different time scales \citep{karniadakis2021physics}, although numerous methods of solutions for this problem have been suggested  \cite{kim2024review}. Challenges have been identified for PDEs featuring strong advection terms \citep{krishnapriyan2021characterizing}. Systems with very little friction, such as Hamiltonian systems, are challenging to simulate using standard PINNs (as well as traditional numerical methods like Runge-Kutta), especially over the long term. An illustrative example of these difficulties can be seen in attempting to solve Kepler’s problem, \emph{i.e.}, the motion of a planet around the Sun. A typical simulation based on a numerical non-structure preserving algorithm such as Runge-Kutta will result in the planet slowly spiraling into the Sun, whereas a naive application of `vanilla' PINNs may lead to a trajectory quickly deviating from the well-known exact solution, which is an ellipse.

\subsection{Canonical and non-canonical Hamiltonian systems}

\paragraph{Canonical Hamiltonian Systems} There have been extensive studies on the application of physics-informed data methods to compute Hamiltonian systems. Most previous work has focused on \emph{canonical} Hamiltonian systems, where the system's motion law (equation \eqref{general_eq}) exhibits a specific structure: $ \mathbf{u} $ is $2n$-dimensional, represented as $ \mathbf{u} =(\bq, \bp)$, and there is a function $H(\bq,\bp)$, the Hamiltonian, such that $\mathbf{f} $ in \eqref{general_eq} becomes 
\begin{equation}
\mathbf{f}  = \mathbb{J} \nabla_{ \mathbf{u} } H \, , \quad  
\mathbb{J} = 
\left( 
\begin{array}{cc}
0 & \mathbb{I}_n  
\\
- \mathbb{I}_n & 0
\end{array}
\right) \,,
\label{canonical_system_gen}
\end{equation}
with $\mathbb{I}_n$ the $n \times n$ identity matrix, 
leading to the \emph{canonical Hamilton equations} for $(\bq, \bp)$: 
\begin{equation}
\dot \bq =  \pp{H}{\bp} \, , \quad 
\dot \bp = - \pp{H}{\bq} \,.
\label{canonical_system}
\end{equation}
If we take an arbitrary function $F(\bq,\bp)$ and compute its rate of change when $(\bq(t),\bp(t))$ satisfy \eqref{canonical_system}, we find that this rate of change is described by the \textit{canonical Poisson bracket}: 
\begin{equation}
\frac{dF}{dt} = \left\{ F, H \right\} = 
\pp{F}{\bq}\cdot \pp{H}{\bp}
- 
\pp{H}{\bq}\cdot \pp{F}{\bp}.
\label{canonical_bracket}
\end{equation}
The bracket defined by \eqref{canonical_bracket} is a mapping sending two arbitrary smooth functions $F(\bq,\bp)$ and $H(\bq,\bp)$ into a smooth function of the same variables. This mapping is bilinear, antisymmetric, behaves as a derivative on both functions, and satisfies the Jacobi identity: for all functions $F,G,H$
\begin{equation}
\{ \{ F, G \}, H \} + \{ \{ H, F \}, G \} + 
\{ \{ G, H \}, F \} =0 \, . 
\label{Jacobi_identity}
\end{equation}

\paragraph{General Poisson brackets}
Brackets that satisfy all the required properties, \emph{i.e.}, they are bilinear, antisymmetric, act as a derivation, and satisfy Jacobi identity \eqref{Jacobi_identity}, but are not described by the canonical equations \eqref{canonical_bracket}, are called general Poisson brackets (also known as \textit{non-canonical} Poisson brackets). The corresponding equations of motion are referred to as non-canonical Hamiltonian (or Poisson) systems. Often, these brackets have a non-trivial null space leading to the conservation of certain quantities known as Casimir constants, or simply  \emph{Casimirs}. In addition, the motion preserves specific invariant subsets, called \textit{symplectic leaves}. The Casimirs and symplectic leaves are properties of the Poisson bracket and are independent of a particular realization of a given Hamiltonian.

\paragraph{First approach: discovering the Hamiltonian system}
There is an avenue of thought that focuses on learning the actual Hamiltonian for the system directly from data. 
This approach was first implemented for canonical Hamiltonian systems in \citep{greydanus2019hamiltonian} under the name of \textit{Hamiltonian Neural Networks (HNN)}, which approximates the Hamiltonian function $H(\bq,\bp)$ by fitting evolution of specific data sequence through the canonical equations described in \eqref{canonical_system}, although without explicitly computing the Hamiltonian. This work was further extended to include adaptive learning of parameters and transitions to chaos \citep{han2021adaptable}. The mathematical foundation ensuring the existence of the sought Hamiltonian function in HNN was derived in \citep{david2021symplectic}. 

In an alternative approach, \citep{cranmer2020lagrangian} utilize the fact that canonical Hamiltonian systems are obtained from Lagrangian systems using the Legendre transform from velocity $\dot{\bq}$ to momentum $\bp$ variables \cite{arnol2013mathematical}. This particular method is known as \textit{Lagrangian Neural Networks (LNNs)}, which approximates solutions to the Euler-Lagrange equations for some Lagrangian function $L(\bq, \dot{\bq})$ instead of the Hamiltonian function $H(\bq, \bp)$. More generally, the idea of learning a vector field for non-canonical Poisson brackets was suggested in \citep{vsipka2023direct}. The primary challenge in this approach was enforcing the Jacobi identity \eqref{Jacobi_identity} for the learned equation structure, which is non-trivial and cannot be done exactly for high-dimensional systems.

In these and other works on the topic, one learns the vector field governing the system, assuming that the vector field can be solved using appropriate numerical methods. However, one needs to be aware that care must be taken in computing the numerical solutions for Hamiltonian systems, especially over long time spans, as conventional numerical methods lead to distortion of conserved quantities such as total energy and, when applicable, momenta, as illustrated by the example of Kepler’s problem mentioned earlier. 
For computing the long-term evolution of systems governed by Hamiltonian equations (canonical or non-canonical), whether the equations are exact or approximated by the Hamiltonians derived from neural networks, one can use variational integrator methods \citep{marsden2001discrete,leok2012general,hall2015spectral} that conserve momenta-like quantities with machine precision. However, these integrators may be substantially more computationally intensive compared to non-structure preserving methods. 

\paragraph{Second approach: Learning the structure-preserving mappings of phase space for canonical systems}
In this paper, we focus on an alternative approach, namely, exploring the learning transformations in phase space that satisfy appropriate properties. Recall that if $\boldsymbol{\phi}( \mathbf{u} )$ is a map in the phase space of equation \eqref{canonical_system} with $ \mathbf{u}  = (\bq, \bp)$, then this  map is called \textit{symplectic} if 
\begin{equation}
\left( \pp{\boldsymbol{\phi}}{ \mathbf{u} } \right)^\mathsf{T} 
\mathbb{J}
\left( \pp{\boldsymbol{\phi}}{ \mathbf{u} } \right) = 
\mathbb{J} \, . 
\label{symplectic_map_def}
\end{equation}
A well-known result of Poincar\'e states that, for \emph{any} Hamiltonian, the flow $\boldsymbol{\phi}_t( \mathbf{u} )$ of the canonical Hamiltonian system \eqref{canonical_system}, which maps initial conditions $ \mathbf{u} $ to the solution at time $t$, is a symplectic map \citep{arnol2013mathematical,MaRa2013}. Several authors have pursued the idea of directly seeking symplectic mappings derived from data, rather than identifying the actual equations of canonical systems and subsequently solving them. In this framework, data are provided at discrete points, and predictions are made only at these discrete points. Specifically, if the initial system has data specified at $t_n = (0, h, 2 h, \ldots, N h)$ for some fixed $h$, we are solely interested in predicting the system's state at these discrete times, rather than at all times as in the first approach. Confining the system to be essentially discrete circumvents the challenges associated with using structure-preserving integrators to simulate the system, unlike the first approach where equations with a specific structure must be discovered. Additionally, this method is normally faster than numerically solving the equations associated with the learned Hamiltonian system using a structure-preserving integrator. 
However, this speed comes with a compromise: the discrete methods do not guarantee that the solution approximates the true equation between the chosen time points, such as at $t = (n+0.5)h$ for some integer $n$ in the previous example, but only at the data points themselves $t_n = n h$.

The paper \citep{chen2020symplectic} introduced \emph{Symplectic Recurring Neural Networks (SRNNs)} which approximate the symplectic transformation from data using formulas derived from symplectic numerical methods. An alternative approach for computing canonical Hamiltonian equations for non-separable Hamiltonians was presented in \citep{xiong2020nonseparable}, where they developed approximations to symplectic steps based on a symplectic integrator introduced by \citep{tao2016explicit}. This approach was termed \textit{Non-Separable Symplectic Neural Networks (NSSNNs)}.

A more explicit computation of symplectic mappings was undertaken in \citep{jin2020sympnets,chen2021data}. 
The first approach introduced in \citep{jin2020sympnets} computes dynamics by composing symplectic maps of a specific functional form with parameters, where each map maintains its symplectic property across arbitrary parameter values; this method was named \emph{SympNets}. The error analysis in that work focuses on proving local approximation error. Another approach \citep{chen2021data} derives mappings directly using a generating function approach for canonical transformations, implemented as \textit{Generating Function Neural Networks (GFNNs)}. In contrast to SympNets, GFNNs provide an explicit estimate of error in long-term simulations. Finally, \cite{burby2020fast} developed \emph{H\'enonNets}, Neural Networks based on H\'enon mappings, capable of accurately learning Poincar\'e maps of Hamiltonian systems while preserving the symplectic structure.  SympNets,  GFNNs, and H\'enonNets demonstrate the ability to accurately simulate the long-term behavior of simple integrable systems such as pendulums or solve challenging problems like Kepler’s problem, which were difficult for the 'vanilla' PINNs. These methods also achieve satisfactory long-term accuracy in simulating chaotic systems like the three-body plane problem. Thus, learning symplectic transformations directly from data shows great promise for long-term simulations of canonical Hamiltonian systems. 

\paragraph{Learning structure-preserving mappings for non-canonical Hamiltonian systems} 
The paper \citep{jin2022learning} built upon SympNets, extending its application to non-canonical Poisson systems. The method used the Lie-Darboux theorem which states that it is possible to transform the non-canonical Hamiltonian (Poisson) equations to \emph{local} canonical coordinates. This method was named \emph{Poisson Neural Networks (PNNs)}. The PNNs then find such a Lie-Darboux transformation using an auto-encoder and use SympNets to provide an accurate approximation of the dynamics. While this method can in principle treat any Poisson system by learning the transformation to the canonical variables and its inverse, we shall note that there are several difficulties associated with this approach. The most important limitation of this method is that the Lie-Darboux transformation is only guaranteed to be local. In fact, in many examples, such as the rotation of a rigid body (evolution on $SO(3)$), it can shown that there is no global  Lie-Darboux mapping. Therefore, finding the mapping using auto-encoders will necessarily involve multiple mappings that must overlap smoothly. It seems difficult to imagine an autoencoder capable of developing such mappings automatically. Moreover, there is no guarantee of conserving Casimirs, as they are preserved only as accurately as the autoencoder allows. 
However, preserving Casimirs is crucial for accurately predicting long-term phase space evolution \citep{Dubinkina2007}, and should therefore be done as accurately as possible.  In the methods we develop in this paper, all Casimirs will be preserved to machine precision. These challenges of the PNN approach will be manifested even more strongly for the problem we consider here, namely, the dynamics of several interacting parts, where one would need to adequately describe the motion of momenta and the Lie group elements. 

The  SympNets and PNN approach was further extended in \citep{bajars2023locally} where volume-preserving neural networks \emph{LocSympNets} and their symmetric extensions \emph{SymLocSympNets} were derived, based on the composition of mappings of a certain type, with some of the examples being applied to Poisson-type systems such as rigid body motion and the motion of a particle in a magnetic field. However, the extension of the theory to more general non-canonical problems was difficult. 

In parallel with the development of integrators for canonical Hamiltonian systems based on canonical transformations in $(\bq, \bp)$, there has also been recent progress in data-based simulations of Poisson systems \citep{vaquero2023symmetry,vaquero2024designing}. These techniques rely on approximate solutions of Hamilton-Jacobi equations, which are generally very challenging. However, the methods developed in these papers allow for finding approximate solutions, with applications to systems such as rigid body motion.  While this direction looks promising, it remains unclear at the moment whether these techniques can achieve sufficiently accurate solutions to Hamilton-Jacobi equations for highly complex problems, such as coupled systems having an incomplete reduction, \emph{i.e.}, involving both momenta and Lie group elements as variables, as we consider here.

A method for addressing a particular class of Poisson systems, known as \emph{Lie-Poisson systems}, was recently developed in 
\citep{eldred2024lie}. Lie-Poisson systems arise from the complete symmetry reduction of a Hamiltonian system, where the dynamics are reduced to momentum only. For instance, the motion of a rigid body around a fixed center of mass, like that of a satellite, remains invariant under all rigid rotations in space. Then, the Lie-Poisson reduction principle \citep{MaRa2013} states that the equations can be written in terms of angular momentum $\bmu$ only. If $H(\bmu)$ is the Hamiltonian, then the equations of motion due to Euler and the corresponding Poisson bracket are given by 
\begin{equation}
\dot{\bmu} = - \pp{H}{\bmu} \times \bmu \, , \quad 
\frac{d F(\bmu)}{d t} = - \bmu \cdot \left( \pp{F}{\bmu} \times \pp{H}{\bmu} \right) : = \left\{ F, H \right\} \, . 
    \label{Euler_eq_PB}
\end{equation}
The bracket $\{ F, H\} $ defined in \eqref{Euler_eq_PB} is an example of the \textit{Lie-Poisson bracket}, and is based on the commutator of the Lie algebra $\mso(3)$ of the rotation group $SO(3)$. Generally, if $ \left\langle  \cdot ,\cdot \right\rangle $ is the pairing between the elements of the Lie algebra $\mathfrak{g}$ (velocities) and its dual $\mathfrak{g}^*$ (momenta), and $[ \cdot, \cdot ]$ is the commutator of the Lie algebra, then the Lie-Poisson bracket is defined as 
\begin{equation}
   \left\{ F, H \right\}_{\rm LP}:=  -\left< \mu , \left[ \pp{F}{\mu} , \pp{H}{\mu}\right]  \right>.
    \label{LP_general}
\end{equation}
It was shown in \citep{eldred2024lie} that for Hamiltonians of a certain type, namely $H( \mu )=h(\left\langle a, \mu \right\rangle )$, the equations of motion given by \eqref{LP_general} can be solved exactly. The paper used a generalization of Poincar\'e theorem for Poisson systems, stating that every flow generated by a Poisson system preserves the Poisson bracket \citep{MaRa2013}, and the fact that any Lie-Poisson equations are explicitly solvable for every test Hamiltonian of the form $H( \mu )=h(\left\langle a, \mu \right\rangle )$. Then, Poisson transformations of phase space were computed that formed the analogs of canonical mappings in \emph{SympNets}; the dynamics were then constructed using a sequence of such transformations on every time step. When the parameters of these transformations were given by a neural network and then used in the dynamics reconstruction, such a method was called \emph{Lie-Poisson Neural Networks}, or \emph{LPNets}. In cases where the parameters are learned directly from data and the mappings themselves form a network, the method is known as \emph{Global LPNets}, or \emph{G-LPNets}. 

The G-LPNets approach, based on defining compositions of Poisson mappings with learned parameters, has shown good results in accuracy and efficiency of calculations for Lie-Poisson systems. Unfortunately, the G-LPNets approach is not directly applicable to systems with interacting parts, as the transformations derived in the earlier work \citep{eldred2024lie} depended on the particular form of the Lie-Poisson bracket. Indeed, G-LPNets relied on the particular form of Lie-Poisson bracket \eqref{LP_general} for the explicit integration of equations of motion for test Hamiltonians and for constructing the mapping. However, this form of the Lie-Poisson bracket is no longer valid when there is incomplete reduction, and there are not just momenta $\mu$ but also group elements present in the Poisson bracket. Therefore, new methods for simulating such systems must be sought. The solution to this problem is presented in this paper.

\subsection{Statement of the problem}

We consider the problem of describing the dynamics of $n$ coupled elements. In all computations and examples of this paper, for simplicity, we consider only $n=2$, although the theoretical background we develop will work for an arbitrary number of connected elements.  While more general values of $n$ can be readily considered, they lead to rather cumbersome expressions which do not add to theoretical understanding.  We denote the combined set of variables in the phase space, for brevity, as $ \mathbf{u} (t)$. The following aspects of the problem are considered known: 
\begin{enumerate}
    \item Data points $\left\{ \mathbf{u} _i= \mathbf{u} (t_i)\right\}_{i=1}^N$ at the given times $(t_0, t_1, \ldots, t_N)$; 
    \item The geometric nature of each element and the symmetry of the system; 
    \item The topology of the system, namely, which elements are connected to which elements.
\end{enumerate}

That information presents the geometric/topological setup of the system, representing the fundamental mathematical information about the system. As we will discuss below, this known information represents the knowledge of the Poisson bracket. On the other hand, any information about the physics of the system, such as the nature of the potential energy in the Hamiltonian, the functional form of the terms, \emph{etc.} are considered unknown. In other words, we consider the knowledge of the problem's topology and Poisson bracket as primary, to be hard-coded exactly in the network, while the knowledge of physics is known only approximately and determined by the learning procedure.

\subsection{Novelty of this work}

\begin{enumerate}
    \item  Through the analysis of general principles of the systems with several interacting parts, we show that we can create transformations based on the flows of test Hamiltonians depending on both the momenta and the group elements; 
    \item We demonstrate that the method is capable of learning the dynamics across the entire phase space and conserving all Casimirs with machine precision, surpassing the accuracy of the ground truth; 
    \item We consider increasingly complex applications of our methods: 
    \begin{enumerate}
        \item We first apply our method to an integrable system of rotation of two interacting rigid bodies about a common axis, showing excellent long-term stability of the method; 
        \item We study arbitrary three-dimensional rotation of two rigid bodies about a common center, demonstrating the optimal possible growth of errors by our methods in accordance with the Lyapunov exponent and bounded oscillatory behavior of energy in the long-term.; 
        \item We learn and simulate the evolution of two connected and interacting $SE(3)$ elements, describing a simple discretization of an elastic rod having only two elements. All eighteen-dimensional phase space dynamics are learned with high accuracy from only 1000 data pairs, exhibiting oscillatory energy behavior and optimal short-term accuracy in line with the Lyapunov exponent.
    \end{enumerate}
\end{enumerate}
In the first example, the number of dimensions of the phase space is three; in the second, nine: and in the third, eighteen. 

We are not currently aware of any data-based method that can describe the dynamics of the whole phase space\footnote{When we mention the dynamics of the 'whole phase space', we, of course, mean only the dynamics of the domain where the data is available. Just as the literature before us, we use the term to discriminate between learning the dynamics of the whole domain with data and, for example, learning the dynamics on the Casimir manifolds or learning a single trajectory.}. with competing accuracy, conserving all Casimirs, and utilizing a correspondingly low number of data points for learning the dynamics on high-dimensional manifolds.

\section{General setting for interacting systems on Lie groups}

In this section, we will provide the background on the dynamics of interacting elements to explain how to build the CLPNets. The general structure of the equations is essential for our derivations, so a brief description of these equations is warranted.

To keep the discussion general enough, we will use the language of Lie groups. 
Suppose the configuration manifold $Q$ of the system is a direct product of $n$ Lie groups, \emph{i.e.}, $Q = G \times G \times \ldots \times G$. In practice, for physical systems such as the discretization of elastic bodies, the Lie group $G$ is the group of rotations and translations $SE(3)$; for robotics and space applications, $G=SE(3)$ or the rotation group $G=SO(3)$.

In general, the Lagrangian of such a system depends on the elements $g_i \in G$ of each Lie group and on the corresponding velocities $ \dot  g _i$. It is therefore  a scalar function of these variables: 
\begin{equation} 
L = L ( g_1, \dot g_1, g_2, \dot g_2, \ldots , g_n, \dot g_n): T(G \times ... \times G) \rightarrow \mathbb{R}\, ,
\label{Lagr_general}
\end{equation} 
defined on the tangent bundle of $G \times ... \times G$.
For the physical systems we are interested in, such as elastic bodies, the Lagrangian is left invariant with respect to the simultaneous action of $G$ on each factor, \emph{i.e.}, we have
\begin{equation}\label{label_G_inv} 
L ( hg_1,h \dot g_1,h g_2, h\dot g_2, \ldots , hg_n, h\dot g_n)= L ( g_1, \dot g_1, g_2, \dot g_2, \ldots , g_n, \dot g_n),
\end{equation}
for all $h \in G$. One can also consider right-invariant systems, where $h$ will be positioned on the right of the $(g_i, \dot g_i)$ in \eqref{label_G_inv}, with only a slight correction to the formulas. However, most physical systems we are interested in are left-invariant.   Physically, left-invariance states that the quantities of interest are formulated in the body frame, as is the case with the evolution of a rigid body, or the laws of elasticity; the right-invariance indicates the laws of motion formulated in the spatial frame, as is the case in fluids.
 A class of Lagrangians with this symmetry is given by expressions of the form
\begin{equation} 
L ( g_1, \dot g_1, g_2, \dot g_2, \ldots , g_n, \dot g_n) = \ell\big(\omega_1, \omega_2, \ldots \omega_n, \{ p_{ij} \} \big)\,, 
\label{Lagr_reduced}
\end{equation} 
where $\omega _i =g^{-1}_i \dot g_i\in \mathfrak{g} $ are the \textit{body velocities} and $p_{ij}=g_i^{-1} g_j \in G$, $i, j = 1, \ldots, n$, $i \neq j$, the \textit{relative Lie group configurations}, with $ \mathfrak{g} $ the Lie algebra of $G$. For instance, for $G=SO(3)$, $ \omega _i \in \mathfrak{so}(3)$ are the body angular velocities and $p_{ij} \in SO(3)$ the relative orientations of the bodies. This is precisely the case we are going to consider in Section~\ref{sec_SO3}.

The equations of motion we derive below are similar to the treatment of dendronized polymer dynamics considered in \citep{gay2012exact}, with the slight generalization that we do not assume the actual tree-like topology and assume arbitrary interactions between the particles. Particular examples of the physical application of this method are: 
\begin{enumerate}
\item The discretization of Cosserat and Kirchhoff's models for elastic rods with $G=SE(3)$ \citep{SiMaKr1988,antman2005general}, via Lie group integrators (\cite{demoures2015discrete}). 
In that case, only the nearest neighbors contribute to the coupling,  \emph{i.e.},  $i = j \pm 1 $ in \eqref{Lagr_reduced}. This is the case we consider in Section~\ref{sec_SE3} for two particles.
\item The discretization of multi-dimensional elastic systems, in which case the indices $i$ and $j$ become multi-indices with the same dimension as the dimension of the discretized system. This is a case we will consider in a follow-up work.
\item Quantum mechanical computations with several quantum particles in a given volume, with $G=SU(2)$. In that case, all particles entangle, so $i,j$ can take arbitrary values. 
\end{enumerate}

\subsection{Equations of motion}

The dynamics of the system is governed by the Euler-Lagrange equations associated with the Lagrangian in \eqref{Lagr_general}. These Euler-Lagrange equations are found from the Hamilton principle
\begin{equation}\label{HP} 
\delta \int_{t_0}^{t_f} L ( g_1, \dot g_1,  \ldots , g_n, \dot g_n) \, {\rm d} t=0\,,
\end{equation}
for arbitrary variations $ \delta g _i $, $i=1,...,n$, vanishing at the temporal extremities $t=t_0,t_f$.

In order to obtain the equations of motion in terms of the variables $ \omega _i$ and $p_{ij}$ used in \eqref{Lagr_reduced}, we follow the approach of Lagrangian reduction by symmetry and rewrite the critical action principle \eqref{HP} in terms of the Lagrangian $\ell$ in \eqref{Lagr_reduced}. The variations of $ \omega _i$ and $p_{ij}$ induced by the variations $ \delta g _i $ are found as
\begin{equation} 
\begin{aligned} 
\de \omega_i & =  \dot \sigma_i+\operatorname{ad}_{\omega_i} \sigma_i  
\\ 
\de p_{ij} & = - \sigma_i p_{ij} + p_{ij} \sigma_j \,,
\end{aligned} 
\label{variations_eqs}
\end{equation} 
where $ \sigma _i \in \mathfrak{g} $ are the symmetry-invariant variations defined by $\sigma_i:= g_i^{-1} \de g_i$ and $\operatorname{ad}_{ \omega _i} \sigma _i=[ \omega _i, \sigma _i]$, see \ref{notations}.
Hence, Hamilton's principle \eqref{HP} induces the variational principle
\begin{equation}\label{reduced_VP} 
\delta \int_{t_0}^{t_f}\ell\big(\omega_1, \ldots \omega_n, \{ p_{ij} \} \big)\, {\rm d} t=0\,,
\end{equation} 
with respect to the constrained variations given in \eqref{variations_eqs} where $ \sigma _i$ are arbitrary curves in $ \mathfrak{g} $ vanishing at the temporal extremities $t=t_0,t_f$.
Taking variations in \eqref{reduced_VP} we get
\begin{equation}
\de \int_{t_0}^{t_f} \ell \mbox{d} t = \int _{t_0}^{t_f}
\sum_i 
\left< 
\pp{\ell}{\omega_i} \, , \, \de \omega_i 
\right> 
+ 
\sum_{i,j; i \neq j}
\left< 
\pp{\ell}{p_{ij}} \, , \, \de p_{ij}  
\right> \mbox{d} t.
    \label{var_critical_action}
\end{equation}
Using \eqref{variations_eqs} and collecting terms proportional to $\sigma_i$, perhaps under appropriate integration by parts, leads to the equations of motion: 
\begin{equation} 
\begin{aligned} 
&\frac{d}{dt} \pp{\ell}{\omega_i} - \operatorname{ad}^*_{\omega_i}
\pp{\ell}{\omega_i}+ \sum_j \left( \pp{\ell}{p_{ij}} p_{ij}^{-1} 
- p_{ji}^{-1}  \pp{\ell}{p_{ji}} \right) =0 
\\ 
&\dot p_{ij} =  - \omega_i p_{ij} + p_{ij} \omega_j \,,
\end{aligned} 
\label{variations_eqs_gen}
\end{equation} 
where the last equation is obtained by differentiating the expression of $p_{ij}=g_i^{-1} g_j$. 
These equations can be simplified somewhat using the fact that $p_{ij}=p_{ji}^{-1}$. The notations used in \eqref{variations_eqs_gen} are detailed in \ref{notations}.

In what follows, we consider $n=2$, which we will use in all examples of the paper, so we can denote $p_{12}=p=g_1^{-1} g_2$ and equations \eqref{variations_eqs_gen} simplify to be 
\begin{equation} 
\begin{aligned} 
&\frac{d}{dt} \pp{\ell}{\omega_1}  - \operatorname{ad}^*_{\omega_1}
\pp{\ell}{\omega_1}+   \pp{\ell}{p} \,  p^{-1} =0 
\\ 
&\frac{d}{dt} \pp{\ell}{\omega_2}  - \operatorname{ad}^*_{\omega_2}
\pp{\ell}{\omega_2}-p^{-1}\pp{\ell}{p} =0 \\
&\dot  p = - \omega _1 p + p \omega _2\,,
\label{eqs_two_groups}
\end{aligned} 
\end{equation}
with $\ell$ the reduced Lagrangian defined by
\begin{equation}\label{red_Lagr_2} 
\ell: \mathfrak{g} \times \mathfrak{g} \times G \rightarrow \mathbb{R} , \quad \ell( \omega _1, \omega _2,p)=L(g_1, \dot  g_1, g_2,\dot  g_2)\,.
\end{equation}

Note that there is an apparent discrepancy between \eqref{variations_eqs_gen} for an arbitrary number of groups and \eqref{eqs_two_groups}, as there is only one term involving $\pp{\ell}{p}$ in \eqref{eqs_two_groups}, not two. That is due to the fact that in \eqref{variations_eqs_gen} we allowed the Lagrangian to depend on both $p_{ij}$ and $p_{ji}$, which is convenient in several applications. However, due to the symmetry of the Lagrangian, it is possible to restrict to consider only $j=i+1$, see \ref{notations} for more details.

\subsection{Hamiltonian formulation and Poisson brackets}


Going to the Hamiltonian description, we define the symmetry-reduced Hamiltonian associated to $\ell$ in \eqref{red_Lagr_2} by the Legendre transform as follows: 
\begin{equation} 
h: \mathfrak{g} ^* \times \mathfrak{g} ^* \times G \rightarrow \mathbb{R} , \quad h( \mu _1, \mu _2, p)=\sum_i \left< \mu_i, \omega_i \right> - \ell ( \omega _1, \omega _2, p)\, ,
\label{Hamiltonian_def}
\end{equation} 
where $ \omega _1$ and $ \omega _2$ on the right-hand side are uniquely determined from the given $ \mu _1$, $ \mu _2$, $p$ by the condition
\begin{equation}\label{hyperregular} 
\frac{\partial \ell}{\partial \omega _1}( \omega _1, \omega _2, p)= \mu _1, \quad  \frac{\partial \ell}{\partial \omega _2}( \omega _1, \omega _2, p)= \mu _2.
\end{equation} 
This definition assumes that $\ell$ is hyperregular, \emph{i.e.}, the map $( \omega _1, \omega _2) \rightarrow ( \mu _1, \mu _2)$ given by \eqref{hyperregular} is a diffeomorphism for all $p$. From \eqref{hyperregular}, we get
\[
\frac{\partial h}{\partial \mu _i}= \omega _i, \quad \frac{\partial h}{\partial p}=- \frac{\partial \ell}{\partial p} 
\]
so that \eqref{eqs_two_groups} can be equivalently in terms of the Hamiltonian $h$ as
\begin{equation} 
\left\{ \, 
\begin{aligned} 
&\dot \mu_1  - \operatorname{ad}^*_{\pp{h}{\mu_1}}
\mu_1 -   \pp{h}{p} \,  p^{-1} =0 
\\ 
&\dot \mu_2  - \operatorname{ad}^*_{\pp{h}{\mu_2}}
\mu_2 +p^{-1}\pp{h}{p} =0 
\\
&\dot p  = - \pp{h}{\mu_1 } p + p \pp{h}{\mu_2}\,.
\label{hamiltonian_eqs_two_groups}
\end{aligned} 
\right. 
\end{equation}
Taking an arbitrary function $f: \mathfrak{g} ^* \times \mathfrak{g} ^* \times G \rightarrow \mathbb{R} $ and computing its time evolution along a solution $\mu _1(t)$, $ \mu _2(t)$, $p(t)$ of \eqref{hamiltonian_eqs_two_groups}, we get \eqref{hamiltonian_eqs_two_groups} in the Hamiltonian form
\[
\frac{d}{dt} f=\{f,h\}\, 
\]
with the non-canonical Poisson bracket
\begin{equation}\label{reduced_Poisson}
\begin{aligned} 
\{f, h\}= &-\left\langle \mu _1, \left[ \frac{\partial f}{\partial \mu _1},  \frac{\partial h}{\partial \mu _1} \right] \right\rangle -\left\langle \mu _2, \left[ \frac{\partial f}{\partial \mu _2},  \frac{\partial h}{\partial \mu _2} \right] \right\rangle \\
& + \left\langle \frac{\partial f}{\partial \mu _1}, \frac{\partial h}{\partial p}  p ^{-1} \right\rangle - \left\langle \frac{\partial f}{\partial \mu _2}, p ^{-1} \frac{\partial h}{\partial p}   \right\rangle\\
& - \left\langle \frac{\partial h}{\partial \mu _1}, \frac{\partial f}{\partial p}  p ^{-1} \right\rangle + \left\langle \frac{\partial h}{\partial \mu _2}, p ^{-1} \frac{\partial f}{\partial p}   \right\rangle.
\end{aligned}
\end{equation}

We have included additional information about the mathematics of coupled systems in \ref{app_sec_math_coupled}. In that Section of the Appendix, the reader will find the mathematical background, useful definitions, and also the general expression for a class of Casimirs of the Poisson bracket \eqref{reduced_Poisson} in terms of the Casimirs of the single (non-coupled) system. We recall that Casimirs are functions that are preserved along the trajectories of all Hamiltonian systems associated with a given Poisson bracket. To keep the level of abstraction down, we will write these Casimirs explicitly for every particular case we consider.  We also recall that, according to a general result on Poisson structures, the phase space of the non-canonical Hamiltonian system \eqref{hamiltonian_eqs_two_groups} is foliated into symplectic leaves that remain invariant under its flow. As we will see, the application of CLPNets preserves these symplectic leaves and the Casimir functions with machine precision.

\section{Defining the building blocks of Neural Network using transformations generated by test Hamiltonians}
\label{sec:general_LPNets}

We use the fact that for any Poisson bracket, any Hamiltonian generates a flow that preserves the Poisson bracket \citep{MaRa2013}. While the exact Hamiltonian for the system is unknown, the pieces of trajectories in the phase space are known. Our goal will be to construct Poisson transformations of the phase space, which, when composed together with certain parameters depending on the phase space, can explain all the observed motion in some average sense.

To derive the transformations, we consider a class of test Hamiltonians and compute the corresponding mappings generated by these Hamiltonians. The main idea is to classify the test Hamiltonians based on the variables they depend on. We only present the cases for two interacting bodies as described in \eqref{hamiltonian_eqs_two_groups}; a generalization to an arbitrary number of bodies can be done similarly.

\paragraph{Hamiltonians depending on the momenta of the first body $\mu_1$} We consider
$h_1 = f_1( \left< a_1, \mu_1 \right>)$, where $f_1(x)$ is an arbitrary function of the scalar argument and $a_1 \in \mathfrak{g} $. The equations of motion \eqref{hamiltonian_eqs_two_groups} become 
\begin{equation}
 \dot \mu_1 = f'_1( \left< a_1, \mu_1 \right>) \operatorname{ad}^*_{a_1} \mu_1 \, , \quad 
 \dot \mu_2 = 0 \, , \quad \dot p = - a_1 p \,.
    \label{h_1}
\end{equation}
Since the argument of the function $f_1$ is conserved along \eqref{h_1}, the term $f'_1( \left< a_1, \mu_1 \right>) $ becomes a constant prefactor. The equation for $\dot \mu_1$ is linear in momenta and is integrable explicitly \citep{eldred2024lie}, with $f'_1( \left< a_1, \mu_1 \right>)$ becoming a tunable parameter dependent on the position $\mu_1$.  Equations \eqref{h_1} leave $\mu_2$ unchanged, $\mu_2=\mu_2(0)$, and the equation for $p$ is again integrable since it is a linear equation in the components of $p$ when the Lie group element $p$ is represented as a matrix. As it turns out, the explicit expressions for $p$ in the cases of the groups $SO(3)$ and $SE(3)$ are given in terms of left and right rotations and translations. Thus, we have an explicit set of transformations $\mathbb{T}_1$ parameterized by $a_1 \in \mathfrak{g} $: 
 \begin{equation}
     \label{h_1_sol}
     (\mu_1(t), \mu_2(t), p(t) ) = \mathbb{T}_1 (a_1, t ) \left[ (\mu_1(0), \mu_2(0), p(0) ) \right]\,.
 \end{equation}
 
\paragraph{Hamiltonians depending on the momenta of the second body $\mu_2$} Similarly, we take $h_2 = f_2( \left< a_2, \mu_2 \right>)$. Equations \eqref{hamiltonian_eqs_two_groups} lead to: 
\begin{equation}
 \dot \mu_2 = f_2'( \left< a_2, \mu_2 \right>) \operatorname{ad}^*_{a_2} \mu_2 \, , \quad 
 \dot \mu_1 = 0 \, , \quad \dot p =   p a_2 \,.
    \label{h_2}
\end{equation}
As above, since $\left< a_2, \mu_2 \right>=$ const, equations \eqref{h_2} are explicitly solvable for an arbitrary Lie group. The momentum $\mu_1$ is unchanged, the momentum $\mu_2$ is evolved according to a linear evolution equation, and the equation for the Lie group element $p$ is linear and can be solved explicitly. The explicit set of transformations for this case is written symbolically as 
\begin{equation}
     \label{h_2_sol}
     (\mu_1(t), \mu_2(t), p(t) ) = \mathbb{T}_2 (a_2, t ) \left[ (\mu_1(0), \mu_2(0), p(0) ) \right]\,.
\end{equation}

\paragraph{Hamiltonians depending on the relative position/orientation $p$} Let us consider $h=F_3(p)$, where $F_3(p)$ is now an arbitrary function of the group element $p$.  Equations \eqref{hamiltonian_eqs_two_groups} lead to the equations: 
\begin{equation}
 \dot \mu_1 = -  \pp{F_3}{p}p^{-1} \, , \quad 
 \dot \mu_2 = p^{-1} \pp{F_3}{p}  \, , \quad \dot p =  0 \,.
    \label{h_3}
\end{equation} 
One can see that equations \eqref{h_3} are also solvable explicitly, for an \emph{arbitrary} scalar function $F_3(a_3, p)$, where $a_3$ is a set of parameters:  
\begin{equation}
  \mu_1 = \mu_1(0) -  \pp{F_3}{p}p^{-1} t \, , \quad 
 \mu_2 =  \mu_2 (0) + p^{-1} \pp{F_3}{p} t  \, , \quad p=p(0) \, . 
    \label{h_3_sol}
\end{equation}
As above, the explicit set of transformations is symbolically written as
\begin{equation}
     \label{h_3_sol_symb}
     (\mu_1(t), \mu_2(t), p(t) ) = \mathbb{T}_3 (F_3,t) \left[ (\mu_1(0), \mu_2(0), p(0) ) \right]\,.
\end{equation}
In what follows, we will use functions $F_3( p)$ that can be expressed as a sum of activation functions of linear combinations of $p$, when the group element $p$ is represented as a matrix. 
 
We will build a neural network learning the dynamics in the whole space out of the transformations $(\mathbb{T}_1, \mathbb{T}_2, \mathbb{T}_3)$. To show how to apply this general theory, we proceed to the considerations of several particular cases.

\paragraph{Ground truth computations} Because of the complexity of the problems we consider, there is no universal method that is capable of computing the solution of all cases while preserving the structure of equations. There exist Lie-Poisson integrators \citep{marsden1999discrete} that preserve Casimirs with machine precision. However, in practice, these integrators have to be implemented separately for each particular case, and their implementation is rather challenging for the case of coupled Lie groups we consider here. Thus, we use a high accuracy BDF (Backward Differentiation Formula) algorithm from \emph{SciPy} package in Python to obtain all data for training and all ground truth solutions. The relative tolerance parameter is set to $10^{-10}$ and absolute tolerance to $10^{-12}$ for all simulations.

\section{Application to coupled rigid bodies}\label{sec_SO3}
\subsection{Problem statement and physical motivation}

Let us first apply this to the case of the dynamics of coupled rigid bodies. The physical picture we have in mind consists of several rigid bodies attached rigidly in space at their centers of mass, but free to rotate about the centers of mass. They interact using, for example, electrostatic interaction. The physical problem is illustrated in Figure~\ref{fig:coupled_rigid_bodies}. The system consists of two rigid bodies $\mathcal{B}_{1,2}$, which can rotate freely about a common point in space. Each rigid body $\mathcal{B}_i$, $i=1,2$, has a moment of inertia $\mathbb{I}_i$ and contains several charges $q_{i,j}$, $j=1, \ldots,m$, which are located in the coordinates $\boldsymbol{\xi}_{i,j}$ in the body frame and are fixed with respect to that particular body $\mathcal{B}_1$ or $ \mathcal{B}_2$. Two charges on each body are shown in this illustration. 
\begin{figure}
    \centering
    \includegraphics[width=0.5\linewidth]{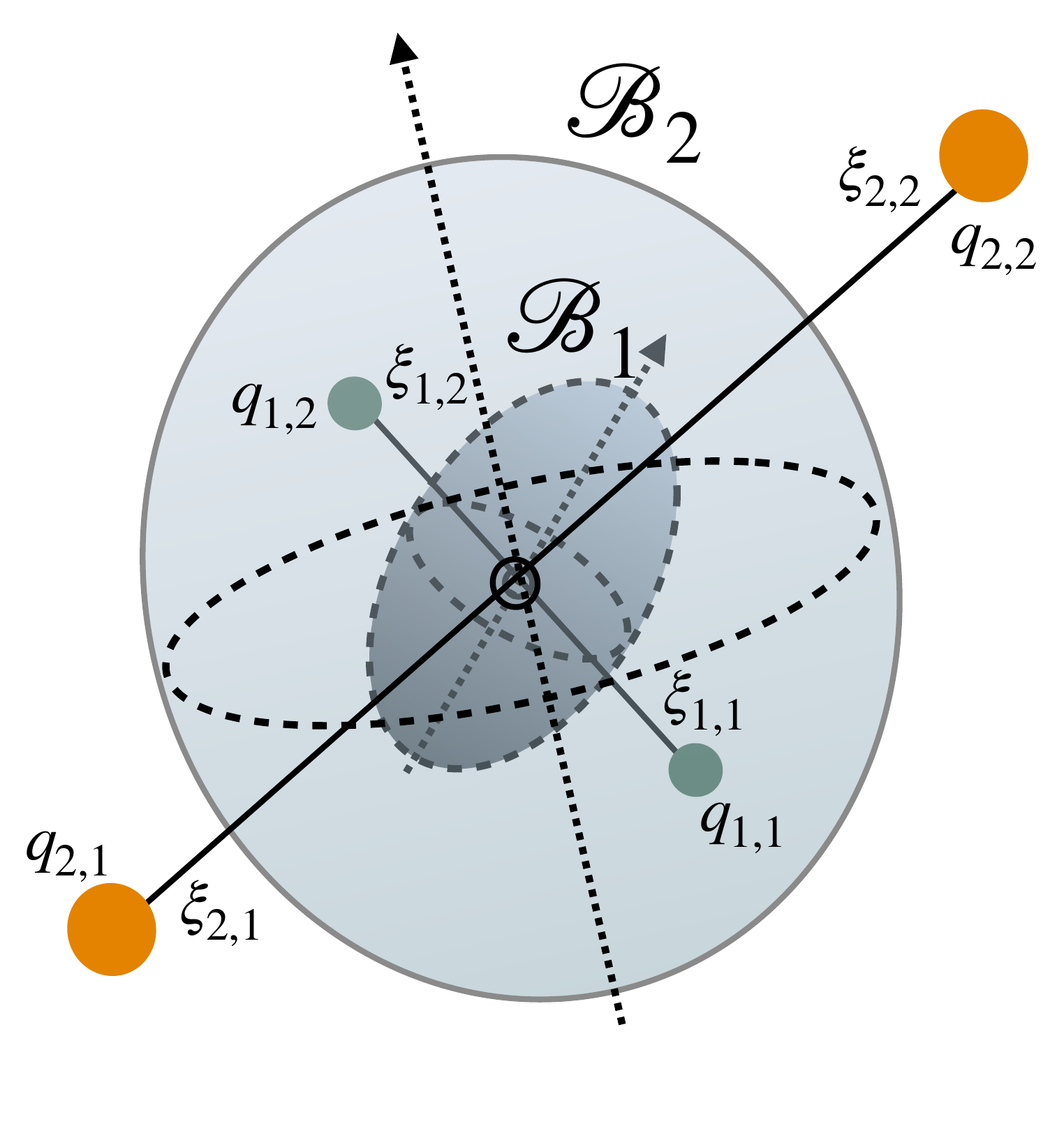}
    \caption{Two rigid bodies $\mathcal{B}_{1,2}$ rotating about a common point which is fixed in space. Each rigid body $i=1,2$ contains several charges $q_{i,j}$, $j=1, \ldots,m$, which are located in the coordinates $\boldsymbol{\xi}_{i,j}$ in the body frame. Two charges on each body are shown on this illustration.  }
    \label{fig:coupled_rigid_bodies}
\end{figure}

\subsection{Equations of motion}

The equations of motion for this system are found by specifying the general system \eqref{hamiltonian_eqs_two_groups} to the Lie group $G=SO(3)$ and taking the Hamiltonian of the coupled rigid bodies, given by the following expression:
\begin{equation}
    h( \bmu_1,\bmu_2, p) = \frac{1}{2} \mathbb{I}_1^{-1}\bmu_1 \cdot \bmu_1 + \frac{1}{2} \mathbb{I}_2^{-2}\bmu_1 \cdot \bmu_2
    + \sum_{i<j} \frac{q_{1,i} q_{2,j}}{d_{ij}(p)}\,,
    \label{SO3_coupled_main}
\end{equation}
see \ref{app_derivation_SO3} for the derivation. These equations are given as: 
\begin{equation}
    \begin{aligned} 
\dot \bmu_1 & = \bmu_1 \times \mathbb{I}_1^{-1} \bmu_1 -  \sum_{i<j} \frac{q_{1,i} q_{2,j}}{d_{ij}(p)^3} \left[ \bxi_{1,i} \times p \bxi_{2,j}  
\right]
\\ 
\dot \bmu_2 & = \bmu_2 \times \mathbb{I}_2^{-1} \bmu_2 +   \sum_{i<j} \frac{q_{1,i} q_{2,j}}{d_{ij}(p)^3} \left[ p^\mathsf{T} \bxi_{1,i} \times  \bxi_{2,j} 
\right] 
 \\ 
 \dot p & = - \widehat{\boldsymbol{\omega} }_1 p + p \widehat{\boldsymbol{\omega } }_2 \, , \quad \boldsymbol{\omega} _i := \mathbb{I}_i^{-1} \bmu_i\,,
    \end{aligned} 
    \label{SO3_coupled_equations}
\end{equation} 
see \eqref{mu_1_mu2_F_gen_so3} in  \ref{app_derivation_SO3}.
Here, we have used the standard notation associating a $3 \times 3$ antisymmetric matrix $u=\widehat{ \mathbf{u}}$ to a $3$-vector $\mathbf{u}$ in such a way that $\widehat{ \mathbf{u} }\, \mathbf{v}= \mathbf{u} \times \mathbf{v}$ for any $\mathbf{v} \in \mathbb{R}^3$. In components, the hat map is 
\begin{equation}
\widehat{ \mathbf{u} } = \left( 
\begin{array}{ccc} 
0 & - u_3 & u_2 \\
u_3 & 0 & - u_1 \\ 
-u_2 & u_1 & 0 
\end{array} 
\right) \, . 
    \label{hatmap_def}
\end{equation}
It is easy to verify from the last equation of \eqref{SO3_coupled_equations} that $\frac{d}{dt} (p^\mathsf{T} p )= [p^\mathsf{T} p, \boldsymbol{\omega }_2]$, so if $p^\mathsf{T}(0)p(0)=Id$, then we have $p(t)^\mathsf{T} p(t)=Id$ for all $t$. This shows that if $p(0) \in SO(3)$, then $p(t) \in SO(3)$ for all times.

For simulations, we take two charges on the body $i=1,2$, positioned at $\pm s_i \mathbf{E}_1$, where $s_1=1$ and $s_2=3$. The charges at $s_i \mathbf{E}_1$ are $q = -0.25$ and the charges at $- s_i \mathbf{E}_1$ are $q=0.25$. The moments of inertia for the bodies are $\mathbb{I} _1= {\rm diag}(1,2,3)$ and $\mathbb{I}_2 = {\rm diag}(2,3,4)$.

\subsection{Casimirs}

Let us consider the more general system 
    \begin{equation}
    \begin{aligned} 
 \dot \bmu_1 & =-\pp{h}{\bmu_1} \times \bmu_1 + \mathbf{F} 
 \\ 
 \dot \bmu_2 & = -\pp{h}{\bmu_2} \times \bmu_2 - p^\mathsf{T} \mathbf{F} 
 \\
 \dot p & = - \pp{h}{\mu_1} p + p \pp{h}{\mu_2} \,,
 \label{mu_1_mu2_F_gen}
 \end{aligned}
    \end{equation}
for an arbitrary Hamiltonian $h$.
Here, $\mathbf{F} = \left(  \pp{h}{p } p^\mathsf{T} \right)^\vee$ in our case, see \eqref{mu_1_mu2_F_gen_so3} in \ref{app_derivation_SO3}, but it could be an arbitrary internal force not necessarily coming from the derivative of $h$ with respect to $p$. 

Let us consider the quantity $\bu = \bmu_1 + p \bmu_2$. Then we have
\begin{equation}
\begin{aligned}
\dot \bu & = -\pp{h}{\bmu_1} \times \bmu_1 + \mathbf{F} 
+ p \left( -\pp{h}{\bmu_2} \times \bmu_2 \right) - \mathbf{F} + \dot p \bmu_2 
\\ 
& = -\pp{h}{\bmu_1} \times \bmu_1  
- p \left( \pp{h}{\bmu_2} \times \bmu_2 \right) - \pp{h}{\bmu_1} \times p \bmu_2 + p \left( \pp{h}{\bmu_2} \times \bmu_2 \right)  
 \\
 & = - \pp{h}{\bmu_1} \times \left( \bmu_1 + p \bmu_2 \right) = - \pp{h}{\bmu_1} \times \bu 
\end{aligned}
\label{dudt_equation}
\end{equation}
Thus, 
\begin{equation} 
\frac{d}{dt }| \bu |^2 = 0 \, , \quad \mbox{for all $h$} \, , 
\label{Casimir_derivation} 
\end{equation} 
so the Casimir has the form 
\begin{equation} 
C ( \boldsymbol{\mu} _1, \boldsymbol{\mu} _2,p)= | \bu|^2 = \left| \bmu_1 + p \bmu_2 \right|^2 = \rm{const} \,.
\label{Casimir_expression} 
\end{equation}  
Physically, $C$ can be understood as the absolute value of the angular momentum seen from either the first or the second rigid body. This Casimir is a particular case of the class given in \eqref{Casimir_coupled}, applied to the Lie group $G=SO(3)$. 

\subsection{Explicit expressions for the Poisson transformations and CLPNets}\label{subsec_4_4}

The algorithm defined in Section~\ref{sec:general_LPNets}, which derives explicit expressions for transformations $\mathbb{T}_{1,2,3}$ forming the neural network, proceeds as follows. Suppose we have $N$ data pairs, each one corresponding to the beginning and the end of the trajectory of length $h$. At each time step, we have the initial and final angular momenta $(\bmu_{1,2}^0,\bmu_{1,2}^f)$ and the relative rotation at the initial and final point $p^0,p^f$. We choose three sub-steps with the time allocated to each step to be $\Delta t/3$, where $\Delta t = h$ is the total time difference between the steps. 
\begin{enumerate}
    \item[{\bf Step 1.}] Take the Hamiltonian to be $h_1 = f_1(\mathbf{a}_1 \cdot \bmu_1)$ for $f_1: \mathbb{R} \rightarrow \mathbb{R} $ and $ \mathbf{a}_1 \in \mathbb{R} ^3  $. The equations of motion are: 
    \begin{equation}
    \dot\bmu_1 = -  f_1'(\mathbf{a}_1 \cdot \bmu_1)\, \mathbf{a}_1 \times \bmu_1 \, , 
    \quad \dot \bmu_2 =0 \, , \quad \dot p = -  f_1'(\mathbf{a}_1 \cdot \bmu_1) \,\widehat{ \mathbf{a}_1} p\, .
    \label{mu1_SO3}
    \end{equation}
    Remember that $\mathbf{a}_1 \cdot \bmu_1=$ const during the motion. The evolution of $\bmu_1$ is a rotation around the axis $\mathbf{a}_1$ with a certain constant angular velocity given by $\omega_1 = f_1'(\mathbf{a}_1 \cdot \bmu_1) |\mathbf{a}_1|$. 
    In order to compute the evolution equation for $p$ on the first step, let us investigate how the change in $p$ affects any vector $\mathbf{v}$. We compute: 
    \begin{equation}
    \frac{d}{dt} p \mathbf{v}=   -  f_1'(\mathbf{a}_1 \cdot \bmu_1)\, \widehat{\mathbf{a}_1} p \mathbf{v} =  -  f_1'(\mathbf{a}_1 \cdot \bmu_1)\,\mathbf{a}_1 \times (p \mathbf{v}) 
    \label{p_eq_comp_1}
    \end{equation}
    so $p \mathbf{v}$ rotates with a given angular velocity about $\mathbf{a}_1$, and thus the result of the rotation is the same as the rotation of $\bmu_1$. This consideration 
    leads to the intermediary values 
    \begin{equation}
    \bmu_1^* = \mathbb{R}(\mathbf{a}_1, \omega_1\Delta t/3) \bmu_1^0 \, , \quad 
    \bmu_2 = \bmu_2^0 \, , \quad  p^* = \mathbb{R}(\mathbf{a}_1, \omega_1 \Delta t/3) p^0\, , 
    \label{step1_SO3}
    \end{equation}
    where $\mathbb{R}(\mathbf{a} _1 , \omega _1\Delta t/3) \in SO(3)$ is the rotation matrix about the normalized axis $\mathbf{a}_1$ by the angle $\varphi_1 = \omega_1 \Delta t/3$.
    
    {\bf Practical implementation:} On each step, we choose three subsequent transformations: $\mathbf{a}_{1,i} = \beta_{1,i} \mathbf{e}_i$, where $\mathbf{e}_i$, $i=1,2,3$ are the basis vectors in the space of $\bmu$. Then, $\mathbf{a}_{1,i} \cdot \bmu_1 = \beta_{1,i} \mu_{1,i}$ and the rotation matrices in \eqref{step1_SO3} are simply rotations about the basis axes $\mathbf{e}_i$, $i=1,2,3$. Each of the transformations corresponds to the choice of Hamiltonians 
    \begin{equation}
    h_{1,i} = f_{1,i} ( \beta _{1,i} \mu_{1,i}) \, , \quad \mbox{with} \quad \pp{h_{1,i}}{\bmu_1} =  \left( \alpha_{1,i} \sigma(\beta_{1,i} \mu_{1,i}) + \gamma_{1,i} \right)\mathbf{e}_i \, , 
        \label{h1_practical_SO3}
    \end{equation}
    where $\alpha_{1,i}$, $\beta_{1,i}$ and $\gamma_{1,i}$ are parameters to be determined and $\sigma(x)$ is the activation function.
    
    \item[{\bf Step 2.}] Take the Hamiltonian to be $h_2 = f_2(\mathbf{a}_2 \cdot \bmu_2)$, for $f_2: \mathbb{R} \rightarrow \mathbb{R} $ and $ \mathbf{a}_2 \in \mathbb{R} ^3$. The equations of motion are: 
    \begin{equation}
    \dot\bmu_1 =  0 \, , \quad \dot \bmu_2 = -  f_2'(\mathbf{a}_2 \cdot \bmu_2)\,\mathbf{a}_2 \times \bmu_2 \, , 
     \quad \dot p =p \,\widehat{\mathbf{a}_2}\,f_2'(\mathbf{a}_2 \cdot \bmu_2)\, , 
    \label{mu2_SO3}
    \end{equation}
    leading to the intermediary values 
    \begin{equation}
    \bmu_1 = \bmu_1^* \, , \quad \bmu^*_2 = \mathbb{R}(\mathbf{a}_2, \omega _2\Delta t/3) \bmu_2^0 \, ,  \quad  p^f = p^*\mathbb{Q}_2(\mathbf{a}_2, \omega _2\Delta t/3) \, , 
    \label{step2_SO3}
    \end{equation}
    where $ \omega _2= f'_2( \mathbf{a}_2 \cdot \boldsymbol{\mu} _2)| \mathbf{a}_2|$ and with $\mathbb{Q}_2(\mathbf{a}_2,h/3)$  the linear operator mapping the initial to the final conditions of the linear ODE $\dot p = p\,\widehat{ \mathbf{a}_2}\,  f_2'(\mathbf{a}_2 \cdot \bmu_2)$. In order to find this operator, we notice that the equation of motion for $p^\mathsf{T}$ is just a rotation since at this step 
        \begin{equation}
    \frac{d}{dt} p^\mathsf{T} \mathbf{v}=   -f_2'(\mathbf{a}_2 \cdot \bmu_2)\,\widehat{ \mathbf{a}_2} p^\mathsf{T}  \mathbf{v} = -  f_2'(\mathbf{a}_2 \cdot \bmu_2)\,\mathbf{a}_2 \times (p^\mathsf{T} \mathbf{v}) \,,
    \label{p_eq_comp_1_gen}
    \end{equation}
    for any vector $ \mathbf{v} $. So $p^\mathsf{T} \mathbf{v}$ is rotating with a fixed angular velocity and thus 
    \begin{equation}
    \mathbb{Q}_2^\mathsf{T}  = \mathbb{R} (\mathbf{a}_2, \omega_2 \Delta t/3)  \, ,  \quad p^f =p^* \mathbb{R}(\mathbf{a}_2, \omega_2\Delta t/3)^\mathsf{T}.
    \label{SO3_solution_LPNets}
    \end{equation}  
    We take $p=p^f$ since at the next step, there will be no more changes in $p$. 
    
{\bf Practical implementation:} On each step, we choose three subsequent transformations: $\mathbf{a}_{2,i} = \beta_{2,i} \mathbf{e_i}$, so $\mathbf{a}_{2,i} \cdot \bmu_2 = \beta_{2,i} \mu_{2,i}$. Each of the transformations corresponds to the choice of Hamiltonians 
    \begin{equation}
    h_{2,i} = f_{2,i} ( \beta _{2,i} \mu_{1,i}) \, , \quad \mbox{with} \quad \pp{h_{2,i}}{\bmu} =  \left( \alpha_{2,i} \sigma(\beta_{2,i} \mu_{2,i}) + \gamma_{2,i} \right) \mathbf{e}_i\, , 
        \label{h2_practical_SO3}
    \end{equation}
    where, again, $\alpha_{2,i}$, $\beta_{2,i}$ and $\gamma_{2,i}$ are parameters to be determined and $\sigma(x)$ is the activation function. 
\item[{\bf Step 3.}] Since $p$ is a $3 \times 3$ matrix, we take the following expression for the Hamiltonian $h=h(p)$: 
\begin{equation}
h(p) = \sum_{i,j} M_{ij}f(p_{ij}) + N_{ij} p_{ij} \, , \quad 
\pp{h}{p_{ij}} = \sum_{i,j}M_{ij}  \sigma(p_{ij}) + N_{ij} \, , 
\label{h_p_LPNets}
\end{equation}
where all the indices run from $1$ to $3$. The explicit form of the transformations for Step 3 is obtained directly from \eqref{h_3_sol} as 
\begin{equation}
\begin{aligned} 
  \mu_1 = \mu_1(0) - \left( \pp{h}{p} p^\mathsf{T} \right)^\vee t \, , \quad 
 \mu_2 =  \mu_2 (0) + \left( p^\mathsf{T} \pp{h}{p}  \right)^\vee t  \, , \quad p=p(0) \, .
\end{aligned}
\label{step3_SO3}
\end{equation}
\end{enumerate}

\rem{ 
Let us take $ h_3(p)= \left\langle A_3, \tau ^{-1} (p) \right\rangle $ with constant $A_3 \in \mathfrak{g} ^* $. We compute $ \frac{\delta h_3}{\delta p}= \left( [{\rm d} \tau ^{-1} _A] ^* A_3 \right) p $ where $A = \tau   ^{-1} (p)$.
The equations for step 3 are
\begin{equation}
\begin{aligned} 
    \dot  \mu _1 & = \frac{\partial h_3}{\partial p} p ^{-1} = [{\rm d} \tau ^{-1} _A] ^* A_3 \, , \quad A:= \tau   ^{-1} (p) 
    \\
    \dot  \mu _2 & = - p ^{-1} \frac{\partial h}{\partial p} = - \operatorname{Ad}^*_p [{\rm d} \tau ^{-1} _A] ^* A_3= - [{\rm d} \tau ^{-1} _{-A}] ^* A_3 
    \\
    \dot p & = 0 
\end{aligned}
\end{equation}
\[
\dot  \mu _2 = - p ^{-1} \frac{\partial h}{\partial p} = - \operatorname{Ad}^*_p [{\rm d} \tau ^{-1} _A] ^* A_3= - [{\rm d} \tau ^{-1} _{-A}] ^* A_3 ,
\]
where the last equality is a property of the Cayley map which arise from $ \tau (-A) \tau  (A)=id$. It is good that the two equations then just differ by changing $A \rightarrow  -A$.
    } 
\rem{ 
 \item[{\bf Step 3}] Take the Hamiltonian $h_3(p)$ such that 
 $\pp{h_3}{p}= p \widehat{A}_3$, where $\widehat{A}_3$ is an element of $\mso(3)$. The equations of motion are then 
 \begin{equation} 
\dot \bmu_1 = \mathbf{A}_3 \, , \quad 
\dot \bmu_2 = p \mathbf{A}_3 \, , \quad \dot p =0 \, . 
\label{p_SO3}
\end{equation} 
\todo{ \textcolor{blue}{ $\dot \bmu_1 = p\mathbf{A}_3$, $\dot \bmu_2 = - \mathbf{A}_3$.}}
In \eqref{p_SO3} we used the fact that $(p \widehat{A}_3 p^{-1})^\cup = p \mathbf{A}_3$ which can be computed as 
\[ 
\left< p \widehat{A}_3 p^{-1} , v \right> = \left<  p, \operatorname{Ad}_{p^{-1}} v \right> =  \mathbf{A}_3 \cdot  p^{-1} \mathbf{v} = p \mathbf{A}_3 \cdot \mathbf{v}
\]
    We thus integrate \eqref{p_SO3} to obtain
    \begin{equation} 
 \bmu_1^f = \mathbf{A}_3 \Delta t/3 + \bmu_1^* \, , \quad 
 \bmu_2^f = p^f \mathbf{A}_3\Delta t/3 + \bmu_1^* \, , \quad p=p^f \, . 
\label{p_SO3_final}
 \end{equation} 
 } 


\subsection{Particular case of \eqref{SO3_coupled_equations} with a common axis of rotation}\label{subsec_2}
Let us now consider the case when both rigid bodies rotate about a common axis, taken to be $\mathbf{e}_3$. This solution is possible if all charges are in the $(\mathbf{E}_1,\mathbf{E}_2)$ plane of both rigid bodies, and the axes $\mathbf{E}_3$ for both rigid bodies are aligned with the axis of rotation $\mathbf{e}_3$. While this solution is indeed a specific solution of the equations \eqref{SO3_coupled_equations}, the application of CLPNets to this problem enables an understanding of how these particular transformations act.

The matrices $g_i \in SO(3)$ are defined as rotations about the $\mathbf{e}_3$-axis: 
\begin{equation}
  g_i =   \begin{pmatrix}
  \cos \varphi_i & - \sin \varphi_i & 0 \\ 
  \sin \varphi_i &  \cos \varphi_i & 0 \\ 
  0 & 0 & 1
    \end{pmatrix} \, , \quad 
    p =  \begin{pmatrix}
  \cos \Phi & - \sin \Phi & 0 \\ 
  \sin \Phi &  \cos\Phi & 0 \\ 
  0 & 0 & 1 
    \end{pmatrix} \, ,   
    \label{p_SO2}
\end{equation}
where we have defined $\Phi:= \varphi_1 - \varphi_2$ to be the relative angle. Since the rotations are about the $\mathbf{e}_3$ axis, and $\bmu_i$ are also parallel to the $\mathbf{e}_3$ axis, the first terms on the right-hand sides of the first two equations of \eqref{SO3_coupled_equations} vanish. All the applied forces are in the $(\mathbf{e}_1,\mathbf{e}_2)$ direction, and therefore all the torques are in the $\mathbf{e}_3$ direction. Therefore, the equations of motion take the form 
\begin{equation}
\dot \mu_3^1 = T(\Phi) \, , \quad \dot \mu_3^2 = - T(\Phi), \quad \dot p = - (\widehat{\boldsymbol{\omega}_1} - \widehat{ \boldsymbol{\omega}_2}) p \, , \quad \boldsymbol{\omega}_i := \frac{1}{I_3^i} \boldsymbol{\mu}_3^i = \frac{\mu_3^i}{I_3^i} \mathbf{e}_3 \, , 
    \label{SO2_equations_init}
\end{equation}
with $T(\Phi)$ being some torque. The indices $i$ above the variable denote which variable it refers to, and the subscript $3$ denotes the $\mathbf{e}_3$ component of each variable. 
A little algebra shows that the equations can be cast into the form 
\begin{equation}
\dot \mu_3^1 = T(\Phi) \, , \quad \dot \mu_3^2 = - T(\Phi),  \quad \dot \Phi\, \mathbb{Q}  = - \left( \frac{\mu_3^1}{I_3^1} -  \frac{\mu_3^2}{I_3^2}  \right) \mathbb{Q} \, ,
    \label{SO2_equations_med}
\end{equation}
with $\mathbb{Q}=\mathbb{Q}(\Phi)$ defined as 
\begin{equation}
    \mathbb{Q}(\Phi) = \pp{p}{\Phi} = 
\begin{pmatrix} 
\sin \Phi & \cos \Phi & 0 \\ 
- \cos \Phi & \sin \Phi & 0 \\ 
0 & 0 & 0 
\end{pmatrix}\, . 
\label{dp_dphi_definitions}
\end{equation}
We can further simplify \eqref{SO2_equations_med} as 
\begin{equation}
\dot \mu_3^1 = T(\Phi) \, , \quad \dot \mu_3^2 = - T(\Phi), \quad \dot \Phi   = - \left( \frac{\mu_3^1}{I_3^1} -  \frac{\mu_3^2}{I_3^2}  \right)   \,.
    \label{SO2_equations_final}
\end{equation}
By differentiating the $\Phi$-equation of the above system once, the equations reduce to a single equation for $\Phi$ computed as 
\begin{equation}
    \ddot \Phi = - \left(  \frac{1}{I_3^1} + \frac{1}{I_3^2} \right) T( \Phi) \,.
    \label{Phi_final_eq}
\end{equation}
The rigid bodies can either rotate about the common axis if the initial momenta are high enough, or wobble around each other initially, so $\Phi$  remains bounded. We shall now explore how to find the same reduced solution in the full space using CLPNets.

\rem{ 
\todo{FGB: I need the strange signs above in $h$ to make it work. Ok? (old version was $h = \frac{\mu_1^2}{I_3^1} + \frac{\mu_2^2}{I_3^2} + U(\varphi_2-\varphi_1)$)\\ 
\textcolor{magenta}{VP: yes, true. Although because this Hamiltonian is quite strange, let us just make a comment that it canonical in $(p_\Phi, \Phi)$ variables which has a positive definite Hamiltonian, so as not to hear (well deserved) static from referees that we consider non-physical Hamiltonians. I corrected the remark below.  }}
} 

\begin{remark}[On the canonical structure of the equations]
\rem{ 
{\rm One can notice that equation \eqref{SO2_equations_final} can be written in the canonical form for $h = -\frac{\mu_1^2}{I_3^1} - \frac{\mu_2^2}{I_3^2} - U(\varphi_1-\varphi_2) $ as
\begin{equation}
    \left\{ 
    \begin{aligned} 
\dot \varphi_i & =  \pp{h}{\mu_i} \, , 
\\ 
\dot \mu_i & =  - \pp{h}{\varphi_i}   \,,
\end{aligned} 
    \right. 
    \label{Canonical_SO2}
\end{equation}
with 
}
} 
{\rm 
One could write \eqref{Phi_final_eq} in the canonical form for the variables $(\Phi, p_\Phi)$ with $p_\Phi = \dot \Phi$ and Hamiltonian $h = \frac{1}{2} p_\Phi^2 +  \Big( \frac{1}{I_3^1}+ \frac{1}{I_3^2}\Big)  U(\Phi)$, where $U'(\Phi)=T(\Phi)$.
Thus, one could use the \emph{SympNets}  \citep{jin2020sympnets} to solve the canonical form of the system \eqref{Phi_final_eq}. However, this form is not generalizable to the arbitrary dynamics of coupled Lie group systems, so we will not use it here.
}
\end{remark}

\subsection{Applications of CLPNets to a particular solution of \eqref{SO3_coupled_equations}}
Note that we are going to use the data that were generated for the general system \eqref{SO3_coupled_equations}, and not the reduced system \eqref{SO2_equations_final}, as to test the ability of our method to find particular solutions of the equations. 
\begin{enumerate} 
\item We test Hamiltonians of the type $h ( \boldsymbol{\mu}  _i)=F(x_i)$, with $ x_i =a_i \,  \bmu_i \cdot  \mathbf{E}_3$, and $F'(x) = f(x) = \alpha \sigma(x) + \gamma$ with $\sigma(x)$ being the activation function. 
Since the rotations about the $\mathbf{e}_3$ axis commute with $p$ given by \eqref{p_SO2}, the LPNets just changes the angle 
\begin{equation} 
\Phi \rightarrow \Phi + ( f(x_1) a_1-f(x_2) a_2) \Delta t \, . 
\label{Phi_change_S02}
\end{equation} 
and leaves the momenta $\mu_{1,2}$ unchanged. 
\item We consider two matrices $\mathbb{M}$ and $\mathbb{N}$ with four parameters in each matrix:
\begin{equation} 
\mathbb{M} = 
\begin{pmatrix}
M_{11} & M_{12} & 0 
\\ 
M_{21} & M_{22} & 0 
\\
0 & 0 & 0 
\end{pmatrix}\, , \quad 
\mathbb{N} = 
\begin{pmatrix}
N_{11} & N_{12} & 0 
\\ 
N_{21} & N_{22} & 0 
\\
0 & 0 & 0 
\end{pmatrix}
\label{M_N_def} 
\end{equation}
and the Hamiltonian: 
\begin{equation}
h(p)  = \sum_{i,j} 
M_{ij}  F(p_{i j}) + N_{ij} p_{ i j} \,,
    \label{Hamiltonian_SO2_p}
\end{equation}
which is a particular case of \eqref{h_p_LPNets} with a restricted form for $M_{ij}$ and $N_{ij}$. Here $F'(x)=\sigma(x)$, with $\sigma(x)$ being the chosen activation function, and thus
\[ 
\pp{h}{p_{ij}} = M_{ij} \sigma(p_{ij}) + N_{ij}\,.
\]   
The evolution equation for this Hamiltonian leaves  $p$ unchanged and modifies $\mu_{1,2}$ as follows:
\begin{equation}
    \mu_1 \rightarrow \mu_1 + \left( \pp{h}{p} p^\mathsf{T} \right)^\vee \Delta t \, , \quad 
    \mu_2 \rightarrow \mu_2 - \left( \pp{h}{p} p^\mathsf{T} \right)^\vee \Delta t , \quad 
    p \rightarrow p \, . 
    \label{mu_change_SO2}
\end{equation}
Thus, the adjustment of the value of momenta $\mu_{1,2}$ occurs in this step. 

\end{enumerate}
Note that the Casimir \eqref{Casimir_expression} is, in this case, given simply as $C = \mu_1 + \mu_2$ and is conserved to machine precision. 

We present the results of the simulations of the vertical components of momenta $\bmu_i \cdot \mathbf{E}_3$ and the meaningful components of the variables $p$ in Figure~\ref{fig:momenta_orientation_SO2}, for a simulation over 5000 steps ($t=500$ with $h=0.1$). We used only 20 trajectories with 51 time steps each, \emph{i.e.}, 1000 data pairs, to learn the dynamics of the whole phase space. The initial conditions for $\mu_{1,3}$, $\mu_{2,3}$ and the initial rotation angle for $p$ are chosen from a uniform random distribution in the interval $(-2,2)$ for the angular momenta and $(0,2 \pi)$ for the angle. The initial trajectory for testing is chosen with initial momenta taken from a uniform distribution in the interval $(-1,1)$ and an initial angle of relative rotation in the interval $(-\pi/2, \pi/2)$. The description of each rotation by the network involves three parameters. The terms in $(M_{i,j},N_{i,j})$ involve eight parameters, so each cycle $1 \rightarrow 2\rightarrow 3$ contains 14 parameters. A network consisting of three repeated applications of these transformations involves 42 parameters. At the beginning of the learning procedure, all parameters are initialized randomly with a uniform distribution in the interval $(-1,1)$.  
\begin{figure}
    \centering
    \includegraphics[width = 1 \textwidth] {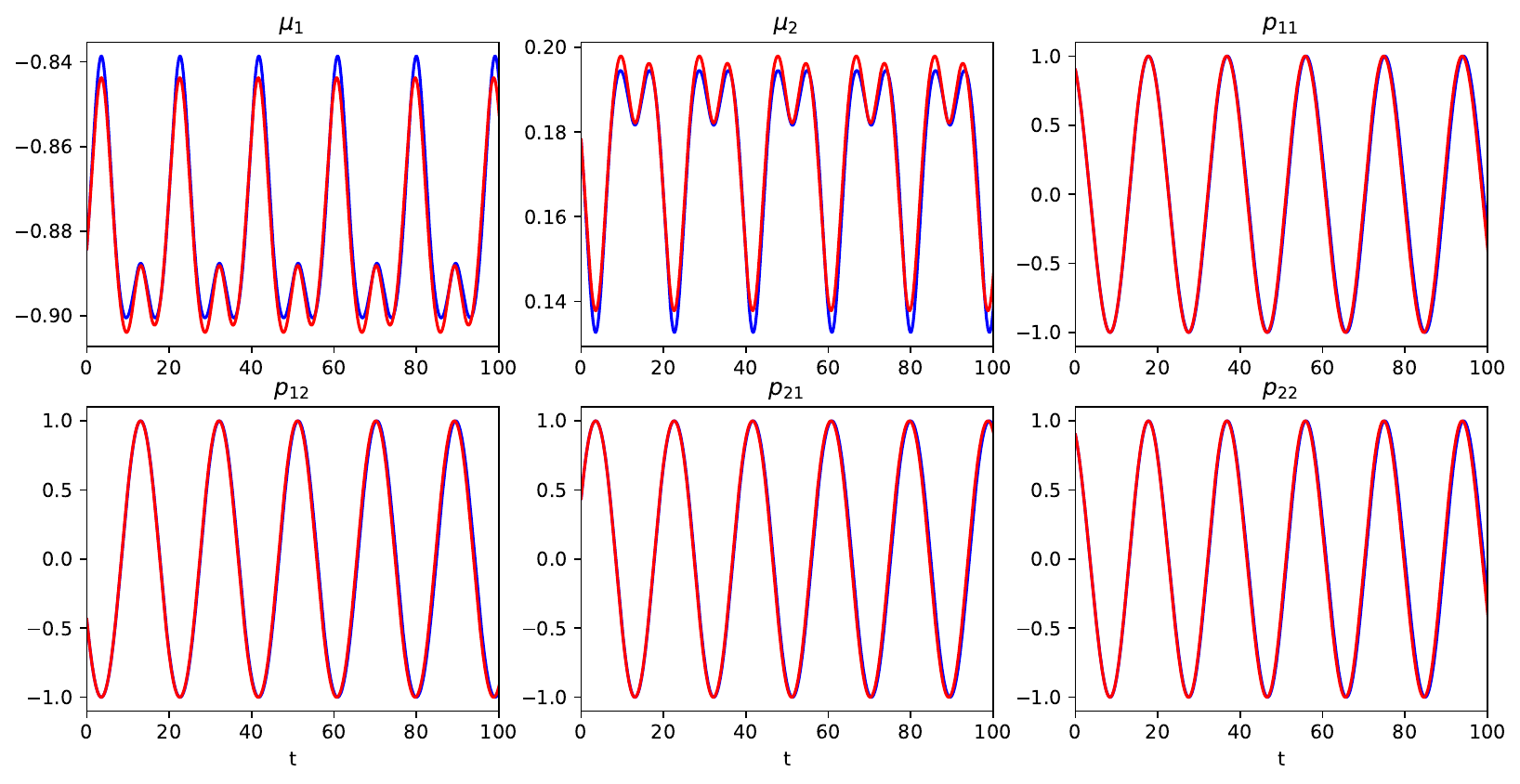}
    \caption{The results of simulations of equations \eqref{SO3_coupled_equations} in the reduced case of $SO(2)$ rotations about the vertical axis. Blue lines: ground truth; red lines: predictions provided by CLPNets.}
    \label{fig:momenta_orientation_SO2}
\end{figure}
In Figure~\ref{fig:Casimir_error_SO2}, we present the comparison of the Casimir from our solution vs the ground truth, the value of the energy, and the value of the mean absolute error. The Casimir is conserved to machine precision in both our case and the ground truth case obtained by the BDF algorithms. It is interesting that since the Casimir is a linear function, namely $C=\mu_1+\mu_2 $, it also happens to be conserved by the BDF algorithm. However, this is an exceptional case and is not true in general. The energy, while not being conserved as well as in the ground truth, remains oscillatory and is preserved well on average.
\begin{figure}
    \centering
    \includegraphics[width = 1 \textwidth]{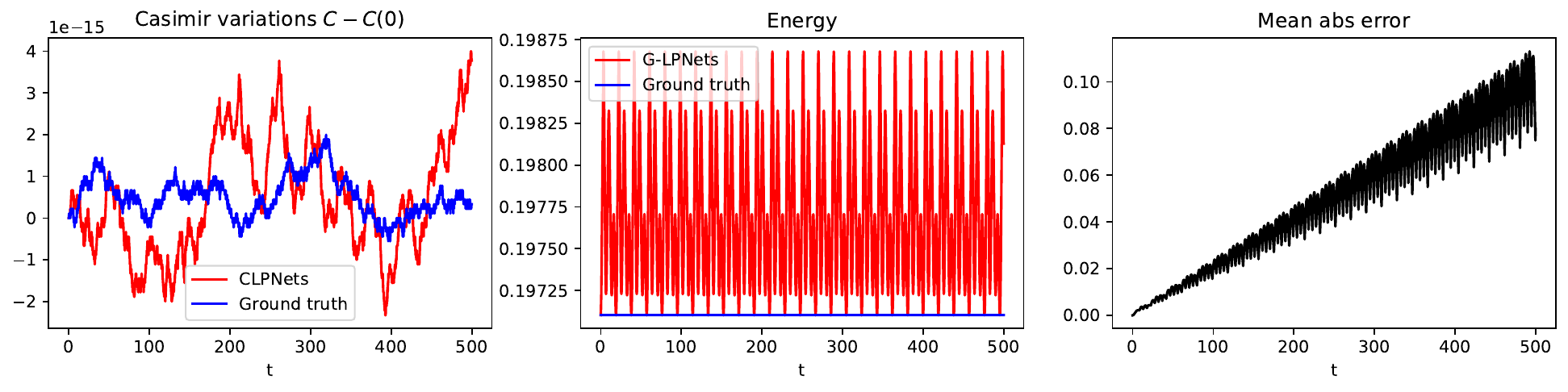}
    \caption{The comparison of the Casimirs (left), energy (center), as well as the computation of the Mean Absolute Error (right) for all components for the simulations presented on Figure~\ref{fig:momenta_orientation_SO2}. The long-term energy conservation for the neural network, although oscillatory, conserves the energy on average with high precision. }
    \label{fig:Casimir_error_SO2}
\end{figure}


\subsection{General dynamics for couple $SO(3)$ groups using CLPNets}

The full three-dimensional motion described by equations \eqref{SO3_coupled_equations} is not integrable and requires the full capability of our network, as defined by the substeps  \eqref{step1_SO3}, \eqref{step2_SO3} and \eqref{step3_SO3}. 
\begin{enumerate}
    \item For the $\bmu_i$-rotation, we utilize rotations around the axis $j$ with the the Hamiltonian 
    \begin{equation}
    h_{i,j}(\bmu_i) = f_{i,j}( \mu_{i,j} ) =  f_{i,j}( \bmu_i \cdot \mathbf{e}_j) \, , \quad 
    \pp{h_{i,j}}{\bmu_i} =  f_{i,j}'(\bmu_i \cdot \mathbf{e}_j)\mathbf{e}_j \, , 
    \label{h_three_dim_LPNets}
    \end{equation}
where $f'_{i,j}(x) =  \alpha_{i,j} \sigma(x)+ \gamma_{i,j}$, with $\sigma(x)$ being the activation function. The motion is simply the rotation about $\mathbf{e}_j$-axis by the angle $f'_{i,j}(x) \Delta t$. 
    \item For the $p$-rotation, we take 
    \begin{equation}
        h(p) = \sum_{i,j,k} M_{ij}f(p_{ij}) + N_{ij} p_{ij} \, , \quad 
        \pp{h}{p_{ij}} = \sum_k M_{ij}  \sigma(p_{ij}) + N_{ij} \,,
        \label{h_p_LPNets_example}
    \end{equation}
\end{enumerate}
where, as usual, $f'(x) = \sigma(x)$ and $\sigma(x)$ is the activation function.
The CLPNets are applied in the following sequence: first, $h(\bmu_1)$ as given by \eqref{h_three_dim_LPNets}, then $h(\bmu_2)$ using the same formula, and finally $h(p)$ as described by \eqref{h_p_LPNets}. This sequence of cycles is repeated $N=3$ times at each time step.  

The learning is obtained from 20 trajectories of 51 points each, corresponding to 1000 data pairs. Note that this is exactly the same number of data points as the reduced $SO(2)$ case described above; however, the $SO(2)$ case is effectively $3$-dimensional, whereas the case of full $SO(3)$ coupled rotations is $9$-dimensional. The initial conditions for these trajectories $\bmu_{1,2}^0$ are chosen at random in a cube of size $4$, so $\mu_{k,i}^0$, $k=1,2$ and $j=1,2,3$ are each taken from a random distribution on the interval $(-2,2)$, while the corresponding initial value of $p_0$ is obtained as a rotation matrix about a random unit vector, with an angle $\varphi$ randomly chosen in the interval $\varphi \in [- \pi/2, \pi/2]$. The total number of parameters of the network in this setting is 108. At the beginning of the learning procedure, all parameters are initialized randomly with a uniform distribution in the interval $(-0.1,0.1)$.  The Adam algorithm is performed for optimization. The loss, taken to be MSE of all components, is reduced from about $3 \cdot 10^{-2}$ to $10^{-6}$ after 2000 epochs in the realization shown. The initial learning rate is $1$, decreasing exponentially to $0.1$. The results of simulations and their comparison with the ground truth are 
shown in Figure~\ref{fig:momenta_orientation_SO3}. 
\begin{figure}
    \centering
    \includegraphics[width = 1 \textwidth] {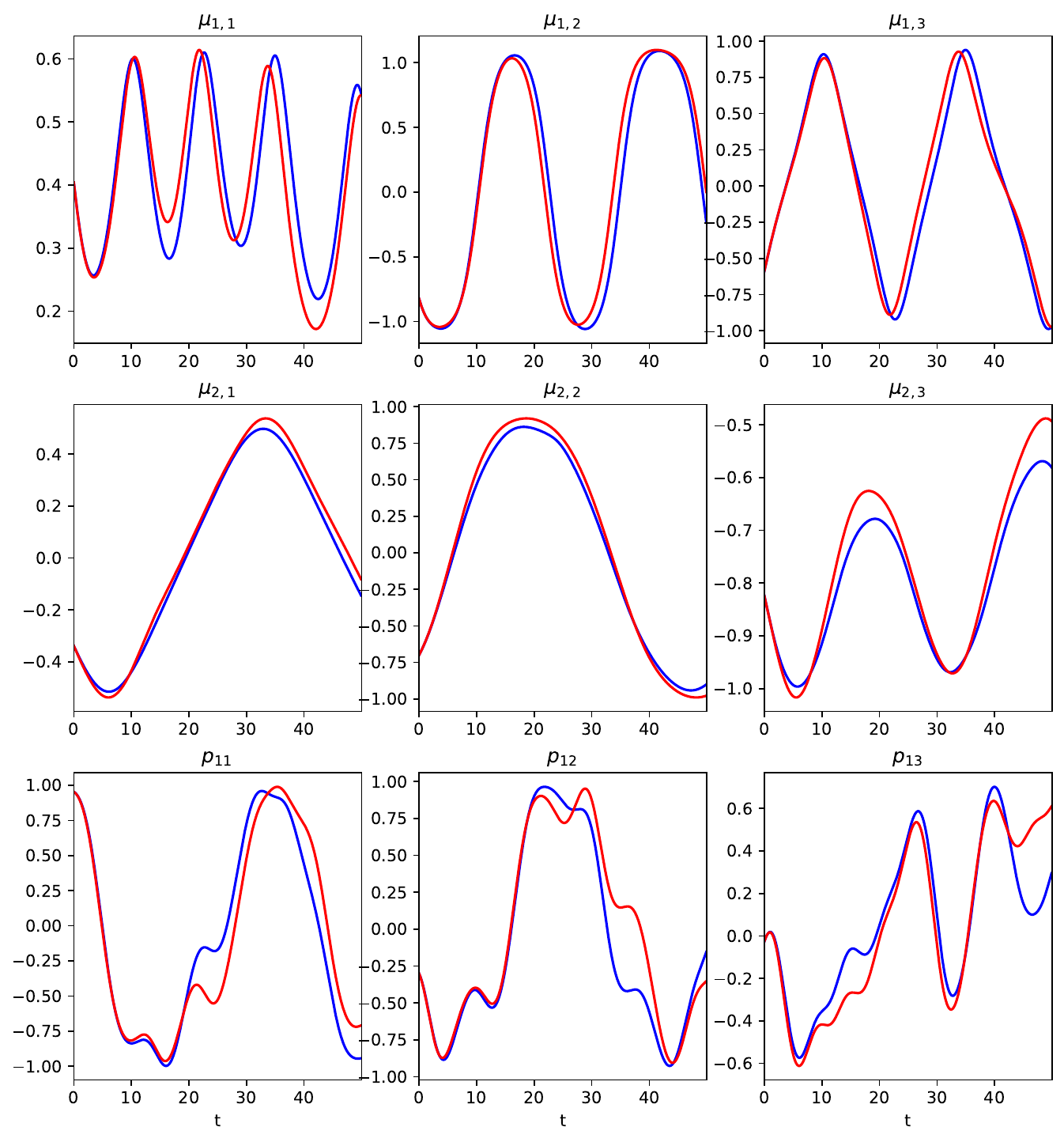}
    \caption{The results of simulations of equations \eqref{SO3_coupled_equations} in the general case of $SO(3)$ rotations. As before, blue lines are the ground truth; red lines are predictions by CLPNets. All components of angular momenta $\bmu_{1,2}$ and three components of orientation matrix $p_{1k}$, $k=1,2,3$ are shown. }
    \label{fig:momenta_orientation_SO3}
\end{figure}
The conservation of Casimirs and energy by the ground truth simulation vs our data is shown on the left and central panels of Figure~\ref{fig:Casimir_error_SO3}, respectively. Note that the Casimir is preserved by our method with machine precision, better than the ground truth. In contrast, the energy is conserved less accurately than the ground truth. For some initial conditions, the preservation of energy is quite accurate on average, although in the case presented here, the error in energy grows approximately linearly in time. Since the solution of equations \eqref{SO3_coupled_equations} are generally chaotic, the distance $d$ between nearby trajectories that are initially close diverges exponentially with time  $d \sim e^{\lambda t}$, where $\lambda$ is the (leading) Lyapunov exponent. We measure the value of the Lyapunov exponent $\lambda$ from direct simulations. The growth of error proportional to $e^{\lambda t}$ is the lowest one can expect in chaotic systems. The growth of mean absolute error (MAE) is demonstrated on the right panel of Figure~\ref{fig:Casimir_error_SO3} as a solid black line, along with the dashed line showing exponential growth according to the Lyapunov exponent $e^{\lambda t}$. 
\begin{figure}
    \centering
    \includegraphics[width = 1 \textwidth]{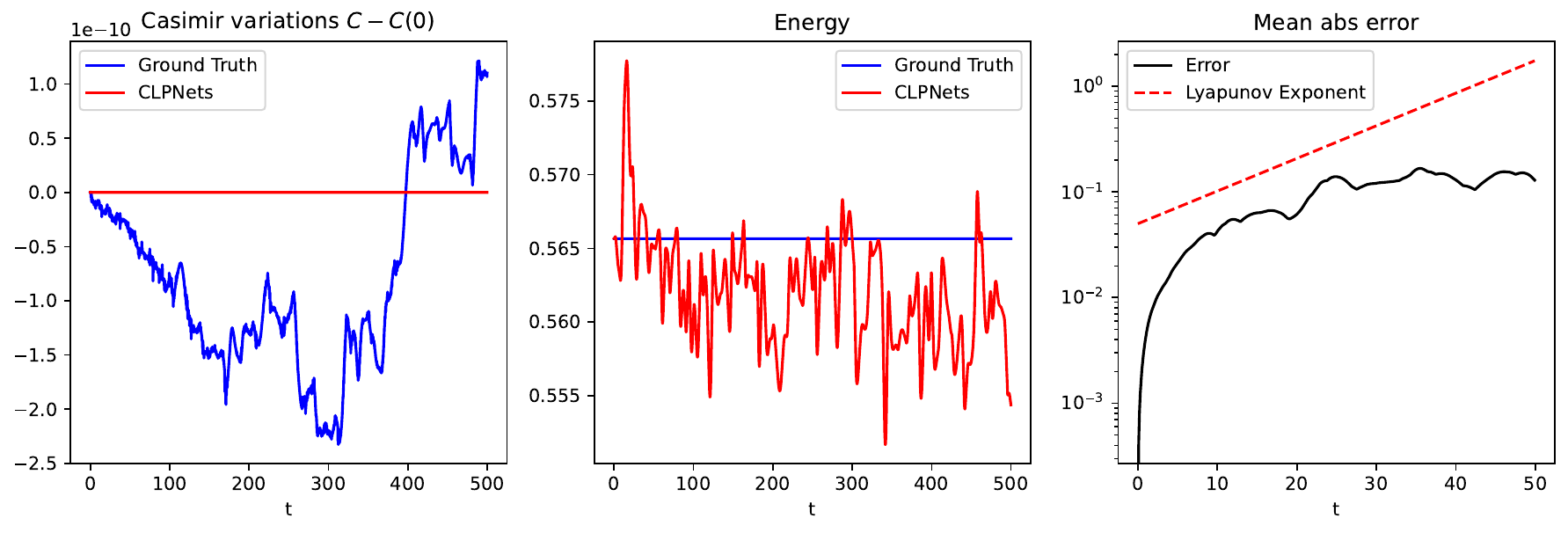}
    \caption{The comparison of the Casimirs (left), energy (center), and the Mean Absolute Error for all components of momenta and three components of $p_{ij}$ for the simulations presented on Figure~\ref{fig:momenta_orientation_SO3}. The energy and Casimir are shown over all the time interval computed (5000 data points, or $t=500$). The MAE on the right panel is computed only for the time interval corresponding to $t=50$ (500 data points). The growth of errors corresponding to the Lyapunov exponent on the right panel is demonstrated with a dashed red line. }
    \label{fig:Casimir_error_SO3}
\end{figure}

\subsection{Results for more general Hamiltonians}


In the Hamiltonian \eqref{SO3_coupled_main}, we used the functional form of the sum of squares of momenta, dictated by physical arguments. To demonstrate that our method is applicable to general Hamiltonians, we simulated the equations with an altered (unphysical) Hamiltonian $h_{\rm alt} = \frac{1}{\zeta}\tanh \zeta h$, where $\zeta$ is a parameter and $h$ is given by \eqref{SO3_coupled_main}. 
Note that each individual trajectory of this new system is obtained by a rescaling in the original system \eqref{SO3_coupled_equations}. However, the time scale depends on the value of that trajectory's Hamiltonian, so the phase flow mapping in the phase space deforms compared to the original system, while still remaining a Poisson transformation. The Casimir is, of course, independent of the Hamiltonian and is still given by the expression \eqref{Casimir_expression}. We present the results of our simulations for $\zeta =0.1$, showing the values of momenta and orientations in Figure \ref{fig:momenta_orientation_SO3_exponentiated}, and the conservation of the Casimir, energy, and error growth in Figure \ref{fig:Casimir_error_SO3_exponentiated}. The results for $\zeta=0.1$ are very similar to those shown in  Figures~\ref{fig:momenta_orientation_SO3} and \ref{fig:Casimir_error_SO3}, with the loss converging to small values, typically about  $5 \cdot 10^{-6}$, within 2000 epochs. This demonstrates that our algorithm can also handle more complex functional forms of the Hamiltonians.
For larger $\zeta$, such as $\zeta=1$, the convergence appears to worsen, with the loss, while decaying, remaining at around $5 \cdot 10^{-5}$ after 2000 epochs, which is about ten times worse than for smaller values of $\zeta$. We attribute this loss of accuracy to large variations in the potential energy terms in \eqref{SO3_coupled_main} due to the dynamics of electrostatic charges. These variations cause significant slowdowns or speedups in solutions with the Hamiltonian $h_{\rm alt} = \frac{1}{\zeta} \tanh \zeta h$ compared to the original Hamiltonian in \eqref{SO3_coupled_main}. This discrepancy in time scales is challenging to capture with the regular activation functions considered here. Due to the lack of physical relevance of this type of Hamiltonian and the challenging nature of the question of capturing dynamics in phase space with different time scales using a neural network, we do not pursue this line of thought further here and defer the discussion to future work on the subject.
\begin{figure}
    \centering
    \includegraphics[width = 1 \textwidth] {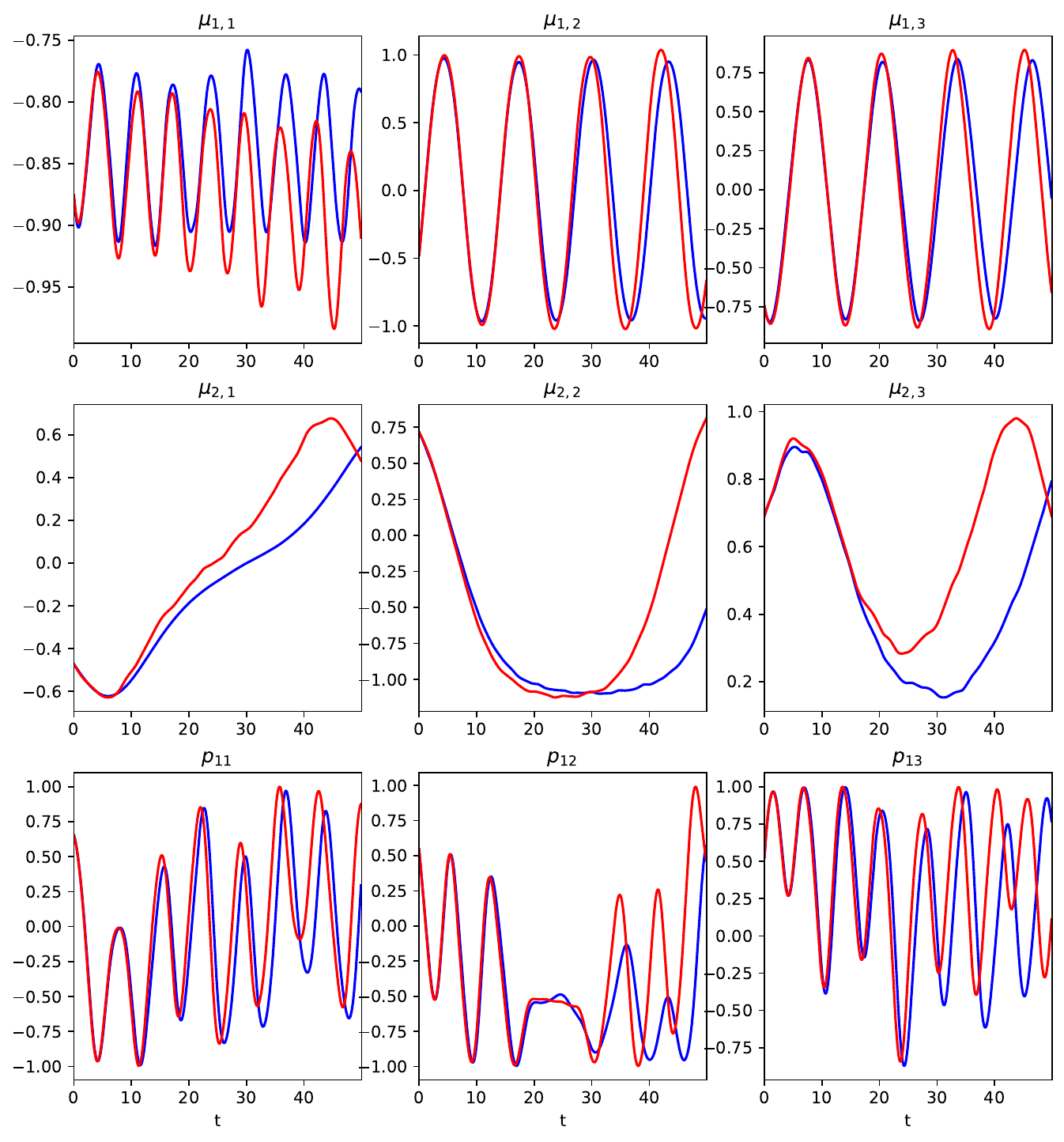}
    \caption{The results of simulations of modified equations \eqref{SO3_coupled_equations} in the general case of $SO(3)$ rotations and $h_{\rm alt} =\frac{1}{\zeta} \tanh {\zeta h}$, where $\zeta = 0.1$ and $h$ is given by \eqref{SO3_coupled_main}. As before, blue lines are the ground truth; red lines are predictions by CLPNets. All components of angular momenta $\bmu_{1,2}$ and three components of orientation matrix $p_{1k}$, $k=1,2,3$ are shown. }
    \label{fig:momenta_orientation_SO3_exponentiated}
\end{figure}

\begin{figure}
    \centering
    \includegraphics[width = 1 \textwidth]{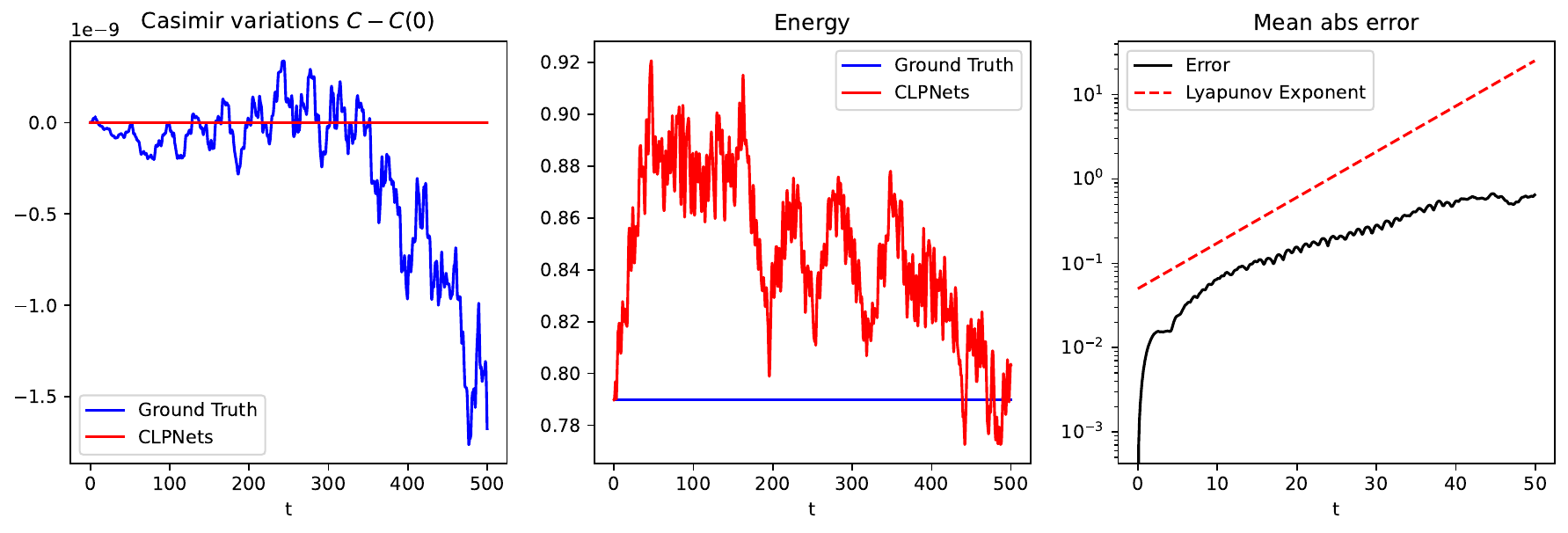}
    \caption{The comparison of the Casimirs (left), energy (center), and the Mean Absolute Error for all components of momenta and three components of $p_{ij}$ for the simulations presented on Figure~\ref{fig:momenta_orientation_SO3_exponentiated}. As before, the energy and Casimir values are shown over all the time interval computed (5000 data points, or $t=500$). The MAE on the right panel is computed only for the time interval corresponding to $t=50$ (500 data points). The growth of errors corresponding to the Lyapunov exponent on the right panel is demonstrated with a dashed red line. }
    \label{fig:Casimir_error_SO3_exponentiated}
\end{figure}

\section{Coupled systems on $SE(3)$: rotational and translational dynamics}
\label{sec_SE3}

\subsection{Lie group definitions, operators and equations of motion}
We now proceed to the consideration of equations of motion for the Lie group of rotations and translations, which is important for the computation of discretization of elastic bodies, such as geometrically exact rods (\cite{demoures2015discrete,gay2016fluid_tube}). This case is the most complex example we will consider. 

We start with some definition of the semidirect product group of rotations and translations $SE(3)$. For a rotation $\mathbb{Q} \in SO(3)$ and translation $\mathbf{v} \in \mathbb{R}^3$, we form a $4 \times 4$ matrix $g \in SE(3)$, with corresponding multiplication and inverse 
\begin{equation}
    g = 
    \left(
    \begin{array}{cc}
    \mathbb{Q} & \mathbf{v} \\
    \mathbf{0}^\mathsf{T} & 1 
    \end{array} 
    \right) 
    \, , \quad 
        g_1 g_2  = 
    \left(
    \begin{array}{cc}
    \mathbb{Q}_1\mathbb{Q}_2 & \mathbb{Q}_1 \mathbf{v}_2 + \mathbf{v}_1 \\
    \mathbf{0}^\mathsf{T} & 1 
    \end{array} 
    \right) 
    \, , \quad 
     g^{-1} = 
    \left(
    \begin{array}{cc}
    \mathbb{Q}^\mathsf{T} & - \mathbb{Q}^\mathsf{T} \mathbf{v} \\
    \mathbf{0}^\mathsf{T} & 1 
    \end{array} 
    \right) \,, 
    \label{SE3_def}
\end{equation} 
where $\mathbf{0}^\mathsf{T}$ is the row of zeros. The momenta of $SE(3)$ are six-dimensional and consist of three-dimensional angular momenta $\balpha$ and linear momenta $\bbeta$. Thus, the evolution equation in the Hamiltonian representation contains the variables $(\balpha_{1,2},\bbeta_{1,2}) \in \mathfrak{se}(3)^*$ and the relative orientation and position $(\mathbb{Q}, \mathbf{v}) \in SE(3)$. The equations of motion are derived from the general expression \eqref{hamiltonian_eqs_two_groups} in \ref{sec_derivation_SE3}. They are given as follows: 
\begin{equation}
\left\{ 
\begin{aligned}
 \dot{\boldsymbol{\alpha}}_1 & = - \pp{h}{\boldsymbol{\alpha}_1}  \times \boldsymbol{\alpha}_1 
 - \pp{h}{\boldsymbol{\beta}_1}  \times \boldsymbol{\beta}_1  + \left( \pp{h}{\mathbb{Q}} \mathbb{Q}^\mathsf{T}\right)^\vee + \mathbf{v} \times \pp{h}{\mathbf{v}}
 \\ 
 \dot{\boldsymbol{\beta}}_1 & = - \pp{h}{\boldsymbol{\alpha}_1}  \times \boldsymbol{\beta}_1 
  + \pp{h}{\mathbf{v}}
 \\ 
  \dot{\boldsymbol{\alpha}}_2 & = - \pp{h}{\boldsymbol{\alpha}_2}  \times \boldsymbol{\alpha}_2
 - \pp{h}{\boldsymbol{\beta}_2}  \times \boldsymbol{\beta}_2 - \left(  \mathbb{Q}^\mathsf{T}\pp{h}{\mathbb{Q}} \right)^\vee 
 \\ 
 \dot{\boldsymbol{\beta}}_2 & = - \pp{h}{\boldsymbol{\alpha}_2}  \times \boldsymbol{\beta}_2
  - \mathbb{Q}^\mathsf{T}  \pp{h}{\mathbf{v}}
\\
\dot{\mathbb{Q}} &  = - \left( \pp{h}{\boldsymbol{\alpha}_1} \right)^\vee \mathbb{Q} + \mathbb{Q} \left( \pp{h}{\boldsymbol{\alpha}_2} \right)^\vee
\\
\dot{\mathbf{v}} & =-  \pp{h}{\boldsymbol{\alpha}_1} 
\times \mathbf{v} -
 \pp{h}{\boldsymbol{\beta_1}}+ \mathbb{Q} \pp{h}{\boldsymbol{\beta_2}}\,.
\end{aligned}
\right. 
\label{SE3_equations_explicit}
\end{equation}
We choose the following Hamiltonian $h$ to compute the training data for \eqref{SE3_equations_explicit}: 
\begin{equation}
\begin{aligned}
h &  = \frac{1}{2} \balpha_1 \cdot \mathbb{I}_1^{-1} \balpha_1 
+
\frac{1}{2 m_1}| \bbeta_1|^2 
+ 
\frac{1}{2} \balpha_2 \cdot \mathbb{I}_2^{-1} \balpha_2
\\ 
& \qquad + 
\frac{1}{2 m_2}| \bbeta_2|^2 
+ \frac{1}{2} \operatorname{Tr} \big( \left( \mathbb{Q}-{\rm Id}_3 \right)^\mathsf{T} \mathbb{P} \left( \mathbb{Q}-{\rm Id}_3 \right) \big) + 
\frac{1}{2} \mathbf{v} \cdot \mathbb{L} \mathbf{v} \,,
\end{aligned} 
    \label{SE3_Hamiltonian_Explicit}
\end{equation}
with the tensors of inertia $\mathbb{I}_1 = {\rm diag}(1,2,3)$, 
$\mathbb{I}_2 = {\rm diag}(2,3,4)$, the masses $m_{1,2}=1$, and the tensors $\mathbb{P} = {\rm diag}(1,2,3)$ and $\mathbb{L} = {\rm diag}(1,2,3)$ associated with the potential energy. 

\subsection{CLPNets for $SE(3)$ dynamics}

For the Lie group $SE(3)$, the algorithm defined in Section~\ref{sec:general_LPNets} proceeds via the following 6 steps.
\begin{enumerate}
    \item[{\bf Step 1.}] Take the Hamiltonian to be $h_1^\alpha = f_1 ( \mathbf{a}_1 \cdot \boldsymbol{\alpha}_1)$, for $  \mathbf{a}_1 \in \mathbb{R} ^3 $. One can readily see that 
    $\mathbf{a}_1 \cdot \boldsymbol{\alpha_1}=$ const. Denoting, as usual, by $\mathbb{R}(\mathbf{a},\varphi)$ the rotation matrix with unit axis $\mathbf{a}$ and angle $\varphi$, the equations of motion are 
\begin{equation}
\left\{ 
\begin{aligned}
 \dot{\boldsymbol{\alpha}}_1 & = - f_1'(\mathbf{a}_1 \cdot \boldsymbol{\alpha}_1) \, \mathbf{a}_1  \times \boldsymbol{\alpha}_1  \quad \Rightarrow \quad \boldsymbol{\alpha}_1 = \mathbb{R}(\mathbf{a}_1,\omega_1 t) \,\boldsymbol{\alpha}_1(0) 
 \\ 
 \dot{\boldsymbol{\beta}}_1 & = - f_1'(\mathbf{a}_1 \cdot \boldsymbol{\alpha}_1) \, \mathbf{a}_1  \times \boldsymbol{\beta}_1    \quad
 \Rightarrow \quad \boldsymbol{\beta}_1 = \mathbb{R}(\mathbf{a}_1, \omega_1t) \,\boldsymbol{\beta}_1(0) 
 \\ 
  \dot{\boldsymbol{\alpha}_2} & = 0  \quad\Rightarrow\quad 
  \boldsymbol{\alpha}_2 = \boldsymbol{\alpha}_2(0)
 \\ 
 \dot{\boldsymbol{\beta}_2} & = 0 \quad
 \Rightarrow\quad 
  \boldsymbol{\beta}_2 = \boldsymbol{\beta}_2(0)
\\
\dot{\mathbb{Q}} &  = - f_1'(\mathbf{a}_1 \cdot \boldsymbol{\alpha}_1) \,\widehat{ \mathbf{a}_1 }\mathbb{Q}  \quad\Rightarrow \quad \mathbb{Q} = \mathbb{R}(\mathbf{a}_1, \omega_1 t) \,\mathbb{Q}(0) 
\\
\dot{\mathbf{v}} & =-f_1'(\mathbf{a}_1 \cdot \boldsymbol{\alpha}_1) \,\mathbf{a}_1 
\times \mathbf{v} \quad \Rightarrow \quad \mathbf{v} = \mathbb{R}(\mathbf{a}_1,\omega_1 t) \,\mathbf{v} (0)\,. 
\end{aligned}
\right. 
\label{SE3_equations_step1}
\end{equation}
Here we defined as before $\omega_1=f_1'(\mathbf{a}_1 \cdot \balpha) | \mathbf{a}_1|$.

{\bf Practical implementation:} On each step, we choose three subsequent transformations: $\mathbf{a}_{1,i} = q_{1,i} \mathbf{e}_i$, where $\mathbf{e}_i$, $i=1,2,3$ are the basis vectors in the $\balpha_1$ space. Again, just as in \eqref{step1_SO3}, the rotation matrices in \eqref{SE3_equations_step1} are simply rotations about the basis axes $\mathbf{e}_i$ in the $\balpha_2$ space, with $i=1,2,3$. Each of the transformations corresponds to the choice of Hamiltonians 
    \begin{equation}
    h_{1,i}^\alpha = f_{1,i} (q _{1,i} \alpha_{1,i}) \, , \quad \mbox{with} \quad \pp{h_{1,i}^\alpha}{\balpha_1} =  \left( s_{1,i} \sigma(q_{1,i} \alpha_{1,i}) + \gamma_{1,i} \right)\mathbf{e}_i \, , 
        \label{h1_practical_SE3}
    \end{equation}
    where $s_{1,i}$, $q_{1,i}$, and $\gamma_{1,i}$ are parameters to be determined and $\sigma(x)$ is the activation function.

    \item[{\bf Step 2.}] Take the Hamiltonian to be $h_1^\beta = g_1 ( \mathbf{b}_1 \cdot \boldsymbol{\beta}_1)$, for $ \mathbf{b}_1 \in \mathbb{R} ^3  $. One can again readily see that 
    $\mathbf{b}_1 \cdot \boldsymbol{\beta}_1=$ const. The equations of motion for this step are 
\begin{equation}
\left\{ 
\begin{aligned}
 \dot{\boldsymbol{\alpha}}_1 & = - g_1' ( \mathbf{b}_1 \cdot \boldsymbol{\beta}_1) \, \mathbf{b}_1  \times \boldsymbol{\beta}_1 \, \quad \Rightarrow \quad \boldsymbol{\alpha}_1 =  - g_1' ( \mathbf{b}_1 \cdot \boldsymbol{\beta}_1) \,\mathbf{b}_1  \times \boldsymbol{\beta}_1 t + \boldsymbol{\alpha}(0) 
 \\ 
 \dot{\boldsymbol{\beta}}_1 & = \mathbf{0}     \quad
 \Rightarrow \quad \boldsymbol{\beta}_1 =  \boldsymbol{\beta}(0) 
 \\ 
  \dot{\boldsymbol{\alpha}_2} & =  \mathbf{0}  \quad\Rightarrow\quad 
  \boldsymbol{\alpha}_2 = \boldsymbol{\alpha}_2(0)
 \\ 
 \dot{\boldsymbol{\beta}_2} & = \mathbf{0} \quad
 \Rightarrow\quad 
  \boldsymbol{\beta}_2 = \boldsymbol{\beta}_2(0)
\\
\dot{\mathbb{Q}} &  = 0  \quad\Rightarrow \quad \mathbb{Q} =  \mathbb{Q}(0) 
\\
\dot{\mathbf{v}} & =g_1' ( \mathbf{b}_1 \cdot \boldsymbol{\beta}_1) \mathbf{b}_1 
 \quad \Rightarrow \quad \mathbf{v} = g_1' ( \mathbf{b}_1 \cdot \boldsymbol{\beta}_1) \mathbf{b}_1 t + \mathbf{v} (0) \,.
\end{aligned}
\right. 
\label{SE3_equations_step2}
\end{equation}

{\bf Practical implementation:} Similarly to Step 1, we take three sequential sub-steps with the choice of vectors $\mathbf{b}_{1,i} = v_{1,i}\mathbf{e}_i $. Then, we choose
\[
h_{1,i}^ \beta = g_{1,i}(v_{1,i} \beta _{1,i}), \quad \text{with} \quad \frac{\partial h_{1,i}^ \beta }{\partial \boldsymbol{\beta} _1}= ( u_{1,i} \sigma( v_{1,i} \beta_{1,i}) + w_{1,i} ) \,\mathbf{e}_i\,,
\]
where $u_{1,i}$, $v_{1,i}$ and $ w_{1,i}$ are parameters to be determined and $\sigma(x)$ is the activation function. 

    \item[{\bf Step 3.}] Take the Hamiltonian to be $h_2 ^\alpha= f_2 ( \mathbf{a}_2 \cdot \boldsymbol{\alpha}_2)$, for $ \mathbf{a}_2 \in \mathbb{R} ^3  $. The equations of motion are 
\begin{equation}
\left\{ 
\begin{aligned}
  \dot{\boldsymbol{\alpha}_1} & = \mathbf{0}  \quad\Rightarrow\quad 
  \boldsymbol{\alpha}_1 = \boldsymbol{\alpha}_1(0)
 \\ 
 \dot{\boldsymbol{\beta}_1} & = \mathbf{0} \quad
 \Rightarrow\quad 
  \boldsymbol{\beta}_1 = \boldsymbol{\beta}_1(0)
\\
 \dot{\boldsymbol{\alpha}}_2 & = - f_2'(\mathbf{a}_2 \cdot \balpha_2)\,\mathbf{a}_2  \times \boldsymbol{\alpha}_2 \quad \Rightarrow \quad \boldsymbol{\alpha}_2 = \mathbb{R}(\mathbf{a}_2, \omega_2 t) \,\boldsymbol{\alpha_2}(0) 
 \\ 
 \dot{\boldsymbol{\beta}}_2 & = - f_2'(\mathbf{a}_2 \cdot \balpha_2)\,\mathbf{a}_2  \times \boldsymbol{\beta}_2    \quad
 \Rightarrow \quad \boldsymbol{\beta}_2 = \mathbb{R}(\mathbf{a}_2,\omega_2 t) \,\boldsymbol{\beta}_2(0) 
 \\ 
\dot{\mathbb{Q}} &  = f_2'(\mathbf{a}_2 \cdot \balpha_2) \,\mathbb{Q} \, \widehat{ \mathbf{a}_2} \quad\Rightarrow \quad \mathbb{Q}  =  \mathbb{Q}(0) \,\mathbb{R}(\mathbf{a}_2, \omega_2t)^\mathsf{T}
\\
\dot{\mathbf{v}} & = \mathbf{0}  \quad  \Rightarrow \quad \mathbf{v} = \mathbf{v} (0) \,,
\end{aligned}
\right. 
\label{SE3_equations_step3}
\end{equation}
with $\omega_2 = f_2'(\mathbf{a}_2 \cdot \balpha_2)$.

{\bf Practical implementation:} Similarly to Step 1,  we choose three subsequent transformations: $\mathbf{a}_{2,i} = q_{2,i} \mathbf{e}_i$, where $\mathbf{e}_i$, $i=1,2,3$ are the basis vectors in the $\boldsymbol{\alpha}_2$ space. Then, $\mathbf{a}_{2,i} \cdot \balpha_2 = q_{2,i} \alpha_{2,i}$. Again, just as in \eqref{step1_SO3}, the rotation matrices in \eqref{SE3_equations_step1} are simply rotations about the basis axes $\mathbf{e}_i$ in the $\balpha_1$ space, with $i=1,2,3$. Each of the transformations corresponds to the choice of Hamiltonians 
    \begin{equation}
    h_{2,i}^\alpha  = f_{2,i} (q_{2,i} \alpha_{2,i}) \, , \quad \mbox{with} \quad \pp{h_{2,i}^ \alpha }{\balpha_1} =  \left( s_{2,i} \sigma(q_{2,i} \alpha_{2,i}) + \gamma_{2,i} \right)\mathbf{e}_i \, , 
        \label{h2_practical_SE3}
    \end{equation}
    where $s_{1,i}$, $q_{1,i}$ and $\gamma_{1,i}$ are parameters to be determined and $\sigma(x)$ is the activation function.
    
    \item[{\bf Step 4.}] Take the Hamiltonian to be $h_2 ^\beta = g_2( \mathbf{b}_2 \cdot \boldsymbol{\beta}_2)$, for $ \mathbf{b}_2 \in \mathbb{R} ^3$. The equations of motion are 
\begin{equation}
\left\{ 
\begin{aligned}
  \dot{\boldsymbol{\alpha}}_1 & =  \mathbf{0}  \quad\Rightarrow\quad 
  \boldsymbol{\alpha}_1 = \boldsymbol{\alpha}_1(0)
 \\ 
 \dot{\boldsymbol{\beta}}_1 & = \mathbf{0} \quad
 \Rightarrow\quad 
  \boldsymbol{\beta}_1 = \boldsymbol{\beta}_1(0)
\\
 \dot{\boldsymbol{\alpha}}_2 & = -  g_2' ( \mathbf{b}_2 \cdot \boldsymbol{\beta}_2)\,  \mathbf{b}_2  \times \boldsymbol{\beta}_2  \quad \Rightarrow \quad \boldsymbol{\alpha}_2 =  -  g_2' ( \mathbf{b}_2 \cdot \boldsymbol{\beta}_2)\, \mathbf{b}_2  \times \boldsymbol{\beta}_2 t + \boldsymbol{\alpha}_2(0) 
 \\ 
 \dot{\boldsymbol{\beta}}_2 & = \mathbf{0}  \quad
 \Rightarrow \quad \boldsymbol{\beta}_2 =  \boldsymbol{\beta}_2(0) 
 \\ 
\dot{\mathbb{Q}} &  = 0  \quad\Rightarrow \quad  \mathbb{Q} =  \mathbb{Q}(0) 
\\
\dot{\mathbf{v}} & = g_2' ( \mathbf{b}_2 \cdot \boldsymbol{\beta}_2)\,  \mathbb{Q} \mathbf{b}_2 \quad \Rightarrow \quad \mathbf{v} =   g_2' ( \mathbf{b}_2 \cdot \boldsymbol{\beta}_2)\, \mathbb{Q} \mathbf{b}_2 t + \mathbf{v} (0) \,.
\end{aligned}
\right. 
\label{SE3_equations_step4}
\end{equation}

{\bf Practical implementation:} Similarly to Step 2, we take three sequential sub-steps with $\mathbf{b}_{2,i} = v_{2,i}\mathbf{e}_i $.  Then, we choose
\[
h_{2,i}^ \beta = g_{2,i}(v_{2,i} \, \beta _{2,i}), \quad \text{with} \quad \frac{\partial h_{2,i}^ \beta }{\partial \boldsymbol{\beta} _2}= ( u_{2,i} \sigma( v_{2,i} \, \beta_{2,i}) + w_{2,i} ) \,\mathbf{e}_i\,,
\]
where $u_{2,i}$, $v_{2,i}$ and $ w_{2,i}$ are parameters to be determined and $\sigma(x)$ is the activation function. 

\item[{\bf Step 5.}] Take the Hamiltonian to be $h_5=h_5(\mathbb{Q})$. The equations of motion are
\begin{equation}
\left\{ 
\begin{aligned}
 \dot{\boldsymbol{\alpha}}_1 & = \left( \pp{h}{\mathbb{Q}} \mathbb{Q}^\mathsf{T}\right)^\vee  
 \quad\Rightarrow\quad 
  \boldsymbol{\alpha}_1 = \left( \pp{h}{\mathbb{Q}} \mathbb{Q}^\mathsf{T}\right)^\vee  t + \boldsymbol{\alpha}_1(0)
 \\ 
 \dot{\boldsymbol{\beta}}_1 & = \mathbf{0} 
 \quad\Rightarrow\quad 
  \boldsymbol{\beta}_1 =  \boldsymbol{\beta}_1(0)
 \\ 
  \dot{\boldsymbol{\alpha}}_2 & =  - \left(  \mathbb{Q}^\mathsf{T}\pp{h}{\mathbb{Q}} \right)^\vee 
  \quad \Rightarrow \quad  \boldsymbol{\alpha}_2 = -\left(  \mathbb{Q}^\mathsf{T}\pp{h}{\mathbb{Q}} \right)^\vee   t + \boldsymbol{\alpha}_2(0) 
 \\ 
 \dot{\boldsymbol{\beta}}_2 & = 0  
  \quad\Rightarrow\quad 
  \boldsymbol{\beta}_2 =  \boldsymbol{\beta}_2(0)
\\
\dot{\mathbb{Q}} &  = 0 \quad 
\Rightarrow \quad \mathbb{Q} =\mathbb{Q}(0) 
\\
\dot{\mathbf{v}} & = \mathbf{0} \quad 
\Rightarrow \quad \mathbf{v} = \mathbf{v}(0) \,.
\end{aligned}
\right. 
\label{SE3_equations_step5}
\end{equation}
{\bf Practical implementation:} For $h_5(\mathbb{Q})$, we choose the neural network parameterization exactly as in \eqref{h_p_LPNets_example}, since we are describing the rotational part of the $SE(3)$ group.
\item[{\bf Step 6.}] Take the Hamiltonian to be $h_6=h_6(\mathbf{v})$. The equations of motion are
\begin{equation}
\left\{ 
\begin{aligned}
 \dot{\boldsymbol{\alpha}}_1 & = \mathbf{v} \times \pp{h}{\mathbf{v}}
 \quad \Rightarrow \quad  \boldsymbol{\alpha}_1= 
 \mathbf{v} \times \pp{h}{\mathbf{v}}t + \boldsymbol{\alpha}_1(0) 
 \\ 
 \dot{\boldsymbol{\beta}}_1 & = \pp{h}{\mathbf{v}}
 \quad \Rightarrow \quad  \boldsymbol{\beta}_1 = \pp{h}{\mathbf{v}}t +  \boldsymbol{\beta}_1(0) 
 \\ 
  \dot{\boldsymbol{\alpha}}_2 & = \mathbf{0} 
 \quad \Rightarrow \quad  \boldsymbol{\alpha}_2 = \boldsymbol{\alpha}_2(0) 
 \\ 
 \dot{\boldsymbol{\beta}}_2 & = \mathbb{Q}^\mathsf{T} \pp{h}{\mathbf{v}}
 \quad \Rightarrow \quad  \boldsymbol{\beta}_2 = \mathbb{Q}^\mathsf{T} \pp{h}{\mathbf{v}} t + \boldsymbol{\beta}_2(0) 
 \\ 
\dot{\mathbb{Q}} &  = 0  \quad 
\Rightarrow \quad \mathbb{Q} = \mathbb{Q}(0) 
\\
\dot{\mathbf{v}} & = \mathbf{0} \quad 
\Rightarrow \quad \mathbf{v}= \mathbf{v}(0) \,.
\end{aligned}
\right. 
\label{SE3_equations_step6}
\end{equation}
{\bf Practical implementation:} We take three sub-steps defined by some functions $V_i$, $i=1,2,3$. Then, we choose
\[ 
h_{6,i}=V_i(L_i \mathbf{e}_i \cdot \mathbf{v}_i)= 
V_i(L_i v_i) \, , \quad \text{with} \quad   
\pp{h_{6,i}}{\mathbf{v}} =  \left( 
M_i \sigma(L_i v_i) + N_i 
\right)\mathbf{e}_i \, , 
\]
where $L_i $, $M_i$ and $N_i$ are some constants and $\sigma(x)$ is the activation function. 
\end{enumerate}

Note that steps 5 and 6 above can be combined into a single step $h = h (\mathbb{Q}, \mathbf{v})$ since the equations for these Hamiltonians are explicitly solvable. 

\subsection{Results of application of CLPNets algorithm for the whole space learning}

The results of the simulations are presented in Figure~\ref{fig:SE3_results}, showing the angular and linear momenta $(\balpha_i,\bbeta_i)$, $i=1,2$, which exhibit reasonable agreement. The learning was conducted with 1000 data pairs, coming from $20$ trajectories of length $51$ data points each, with initial conditions randomly positioned in phase space. Specifically, the initial conditions for the trajectory data pairs $(\balpha_i,\bbeta_i)$ are chosen randomly within the cube $|\balpha_i| \leq 2$, $|\bbeta_i| \leq 2$; the initial rotation for the data pair $\mathbb{Q}$ is a rotation about a random axis with an angle randomly selected in the interval $|\phi_i| \leq \pi/2$, and the initial conditions for $\mathbf{v}$ are also chosen randomly within the cube $|\mathbf{v}| \leq 0.5 \pi$. All random distributions are uniform. The ground truth solutions are computed using a high-precision BDF algorithm for $\Delta t=h=0.1$. 

The CLPNets algorithm is comprised of transformations of types 1-6 described above. The sequence of transformations performed is $1 \rightarrow 2 \rightarrow 3 \rightarrow \ldots \rightarrow 6$. This sequence is repeated three times, resulting in a total of 189 parameters. The Adam algorithm is used for optimization, resulting in a reduction of the loss, which is taken to be the Mean Square Error (MSE), from about $0.02$ to $7\cdot10^{-5}$ in 2000 epochs for the example shown here. The initial learning rate is chosen to be $1$, decreasing exponentially to $0.1$ by the end of the learning procedure. The values of all parameters of the neural network are initialized randomly with a uniform distribution in the range $(-0.1, 0.1)$.

After the learning is completed, we perform simulations starting with a random initial point, which we have chosen at random in a cube $|\balpha_i| \leq 2$, $|\bbeta_i| \leq 1$; the initial rotation for the data pair $\mathbb{Q}$ is again taken to be a rotation about a random axis with an angle randomly chosen in the interval $|\phi_i| \leq \pi/4$, and initial conditions for $\mathbf{v}$ are chosen randomly from the cube $|\mathbf{v}| \leq 0.5$. Thus, this particular trajectory was never observed by the system during learning, and our algorithm is capable of learning the dynamics in the whole phase space.

\begin{remark}[On resisting the curse of dimensionality by CLPNets]
{\rm We remind the reader that the phase space for the coupled \( SE(3) \) system considered here is 18-dimensional, with each momentum \((\alpha_i, \beta_i)\) belonging to a linear space of dimension 6, and the \( SE(3) \) group elements \((Q, v)\) also described by 6 parameters. It is thus even more surprising that the dynamics of the whole phase space can be described by such a low number of data points. Thus, one can say that our method resists the curse of dimensionality. We conjecture that the method can be highly computationally efficient for higher-dimensional problems, such as those arising from the discretization of realistic elasticity models.}
 \end{remark}
 \begin{figure}
    \centering
    \includegraphics[width=0.8\textwidth]{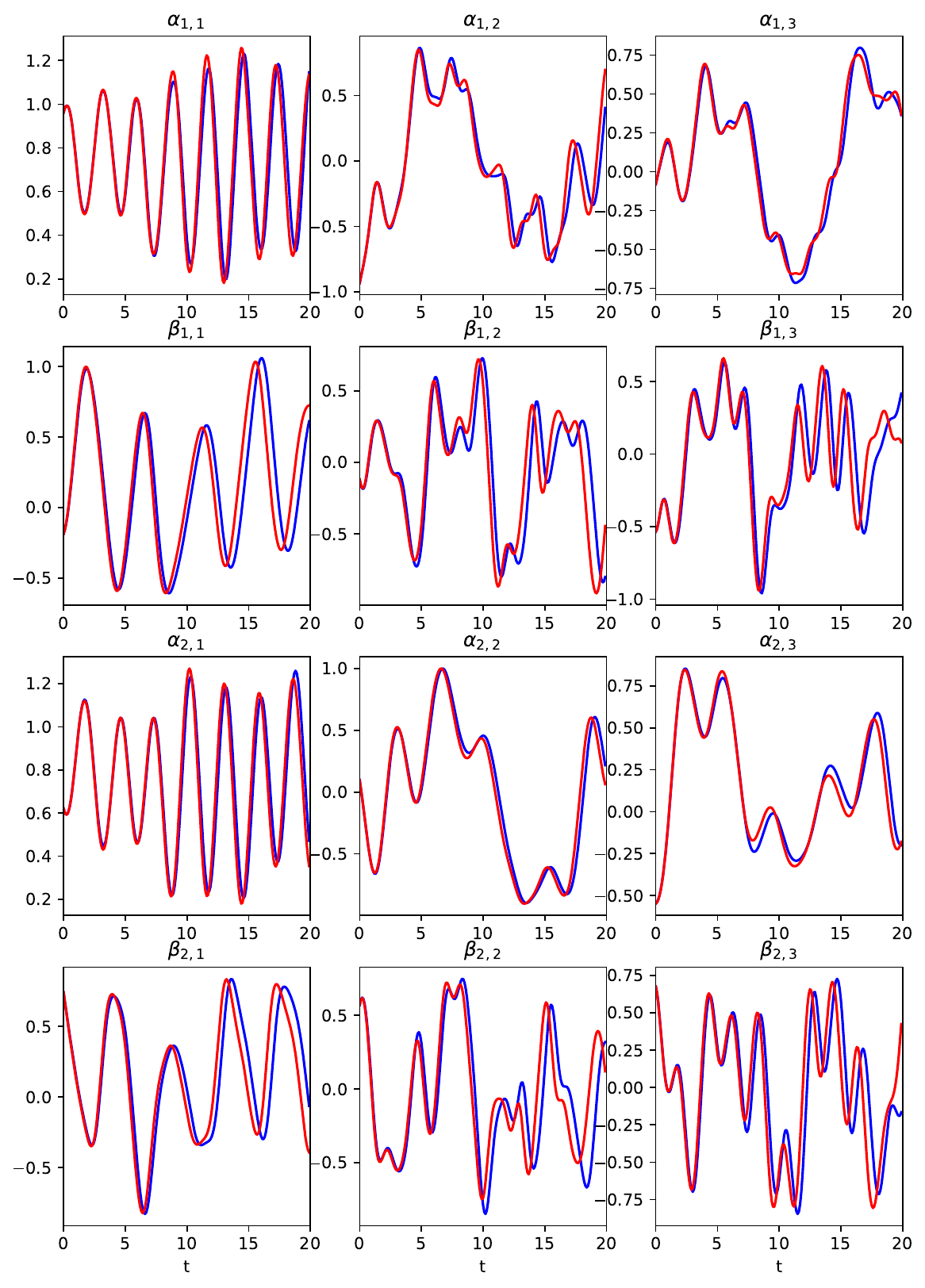}
    \caption{Prediction for a solution of \eqref{SE3_equations_explicit}, with 1000 data pairs used for learning. All components of momenta $(\boldsymbol{\alpha}_{1,2}, \boldsymbol{\beta}_{1,2})$ are shown; the comparison of the group components $(\mathbb{Q},\mathbf{v})$ is not shown for brevity. Again, the blue lines represent the ground truth whereas the red lines represent the predictions provided by CLPNets. }
    \label{fig:SE3_results}
\end{figure}
While we do not observe a perfect match like we did for some simpler integrable systems, we note that since the system is chaotic, a perfect match should not be expected. On Figure~\ref{fig:Casimir_Energy_Error_SE3}, 
we show the results of the conservation of the Hamiltonian and the Casimir (left panel) and the error growth vs Lyapunov's exponent (right panel). The error growth matches the results of Lyapunov's exponent, which is the best result one can hope to achieve for the prediction of a chaotic system. 

While the results on Figure~\ref{fig:SE3_results}   and Figure~\ref{fig:Casimir_Energy_Error_SE3}  do not show perfect matching expected for more simple systems, we remind the reader that the learning was done on an 18-dimensional manifold using only 189 parameters and 1000 data points. Thus, we find the application of CLPNets to the coupled $SE(3)$ system to be quite effective. 

On the top part of the figure, we present the conservation of two Casimirs for the system. They are obtained from the general class of Casimirs given in Lemma from \ref{Lemma_Casimir} following 
\citep{krishnaprasad1987hamiltonian}: if $C_ \mathfrak{g} ( \mu )$ is a Casimir of the motion on just the Lie algebra part of the dynamics, then $C( \mu _1, \mu _2, p)=C_ \mathfrak{g} ( \mu _1 + \operatorname{Ad}^*_{p ^{-1} } \mu _2)$ is a Casimir for the extended part of the dynamics. In our case, if $C_ \mathfrak{g} ( \boldsymbol{\alpha }, \boldsymbol{ \beta })$ is a Casimir of the $ \mathfrak{g} =\mse(3)$-dynamics, then 
$C( \boldsymbol{\alpha }_1, \boldsymbol{ \beta }_1, \boldsymbol{\alpha }_2, \boldsymbol{ \beta }_2, \mathbb{Q}, \mathbf{v} )=C_ \mathfrak{g} \big(( \boldsymbol{\alpha }_1, \boldsymbol{ \beta }_1) + \operatorname{Ad^*}_{( \mathbb{Q}, \mathbf{v} )^{-1}} ( \boldsymbol{\alpha }_2, \boldsymbol{ \beta }_2)\big)$ will be the corresponding Casimir for the full dynamics. 

\subsection{Casimirs} 
The regular Lie-Poisson motion on $SE(3)$ has two Casimirs, as is well established in the corresponding theory related to Kirchhoff's equations modeling underwater vehicles \citep{LeMa1997stability,holmes1998dynamics}. In terms of the angular $\balpha$ and linear $\bbeta$ momenta, these two Casimirs are written as 
 $C_1 = \balpha \cdot \bbeta$ and $C_2 = | \bbeta|^2$. 
The general formula \eqref{Casimir_coupled} yields two Casimirs for the coupled $SE(3)$ bodies:  
\begin{equation}\label{Casimir_extended_SE3}
\begin{array}{l}
C_1  = \overline{\balpha} \cdot \overline{\bbeta} 
\\
C_2  = \big| \overline{\bbeta}\big|^2
\end{array}
\qquad \mbox{where} \qquad 
\begin{array}{l}
 \overline{\balpha}   :=  \balpha_1 +\mathbb{Q} \balpha_2 + \mathbf{v} \times  \mathbb{Q} \bbeta_2  
\\
\overline{\bbeta}  := \bbeta_1 + \mathbb{Q} \bbeta_2 \,.
\end{array}
\end{equation}
These Casimirs are plotted in the top panels of Figure~\ref{fig:Casimir_Energy_Error_SE3}. Once again, the Casimirs in our scheme are conserved with machine precision, whereas the Casimirs provided by the ground truth are only conserved up to $10^{-8}$ accuracy. In the bottom panels, we plot the energy conservation (left panel) and the growth of error (right panel). We see that the energy is conserved by our scheme with less precision than the ground truth, as expected, although the energy in this particular case is conserved quite well on average. Additionally, the growth of MAE in all components is consistent with the expected best case given by the Lyapunov exponent. The exact nature of conservation and growth of errors in components and energy depends on the specific point chosen in the phase space, but the Casimirs in our case are always conserved with machine precision, much better than the ground truth solutions. 
\begin{figure}
    \centering
    \includegraphics[width=1 \textwidth]{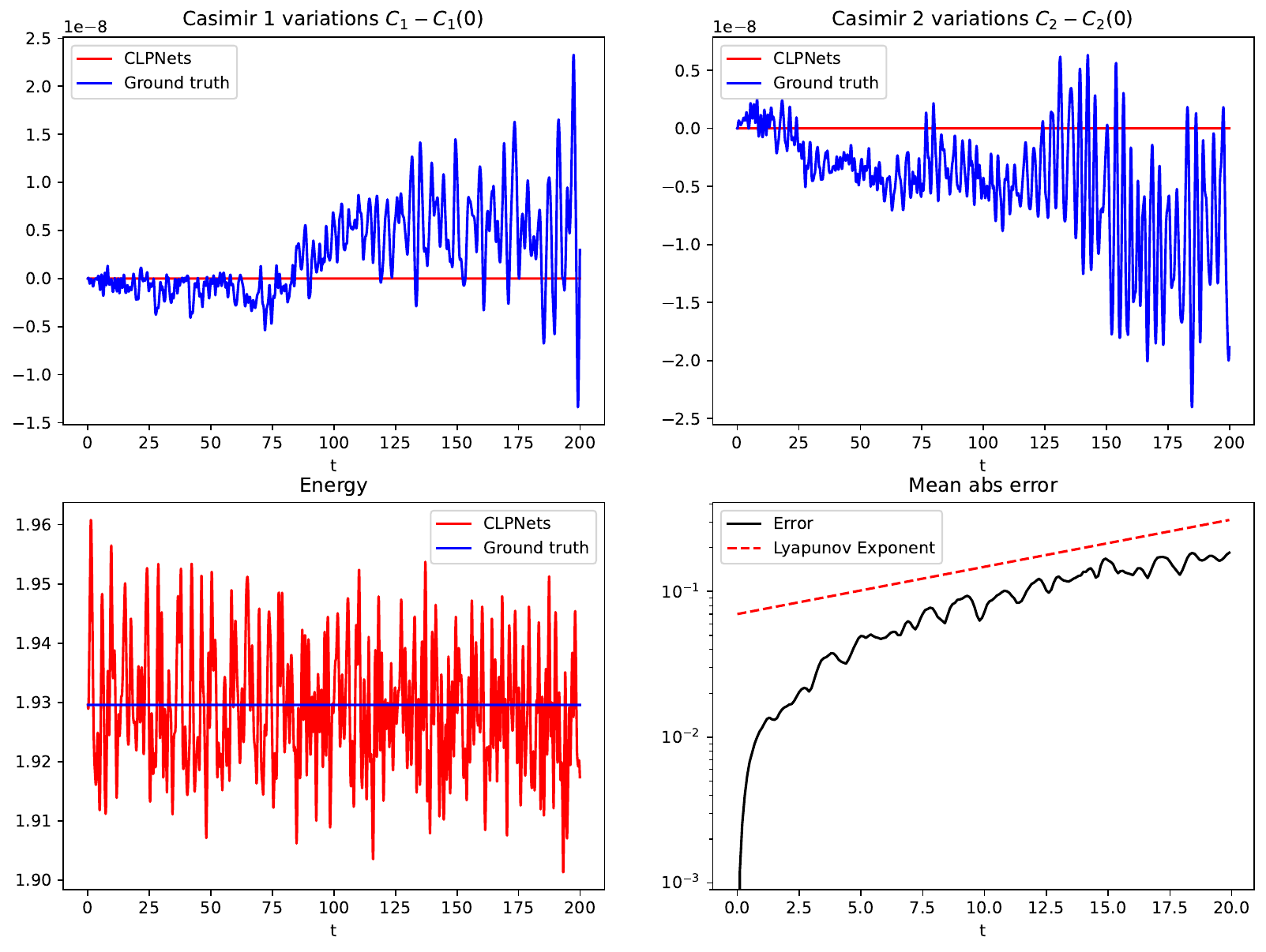}
    \caption{Top: Conservation of the Casimirs. Left/Right panel: Casimirs $C_1$ and $C_2$ as defined by \eqref{Casimir_extended_SE3}. Bottom left panel: conservation of energy by the system. Bottom right panel: Growth of Mean Absolute Error (MAE) in components vs $e^{\lambda t}$, where $\lambda$ is the first Lyapunov exponent. }
    \label{fig:Casimir_Energy_Error_SE3}
\end{figure}

\rem{ 

} 

\section{Conclusions and future work}

In this paper, we have derived a novel algorithm -- CLPNets -- that is capable of learning the whole space dynamics (more precisely, the dynamics of a large portion of the phase space) for systems containing several interacting parts. The learning procedure involves constructing a network from a sequence of transformations with parameters. These transformations, derived from specific functional forms of test Hamiltonians and taken in sequence, faithfully reproduce the structure of dynamics in the phase space.
We demonstrate that learning the dynamics in phase space requires only a modest number of data points,  and the network includes a relatively small number of parameters, even when the system's total dimension is high. The efficient operation of the CLPNets network enables future expansion of this method to compute highly efficient, structure-preserving, data-based solutions to challenging problems, such as those in nonlinear elasticity and continuum mechanics. In particular, in follow-up work, we will study a model of an elastic rod, which consists of a chain of coupled $SE(3)$ elements as described in Section~\ref{sec_SE3}. It will also be interesting to see if alternative data-based, structure-preserving methods suggested recently \citep{vaquero2023symmetry,vaquero2024designing} can offer similar performance for the higher-dimensional coupled systems considered here.

\section{Acknowledgements}
We are grateful to Anthony Bloch, Pavel Bochev, Eric Cyr, David Martin de Diego, Sofiia Huraka, Justin Owen, Brad Theilman, and Dmitry Zenkov for fruitful and engaging discussions. FGB was partially supported by a start-up grant from the Nanyang
Technological University. VP was partially supported by the NSERC Discovery grant. 

This article has been co-authored by an employee of National Technology \& Engineering Solutions of Sandia, LLC under Contract No. DE-NA0003525 with the U.S. Department of Energy (DOE). The employee owns right, title, and interest in and to the article and is responsible for its contents. The United States Government retains and the publisher, by accepting the article for publication, acknowledges that the United States Government retains a non-exclusive, paid-up, irrevocable, world-wide license to publish or reproduce the published form of this article or allow others to do so, for United States Government purposes. The DOE will provide public access to these results of federally sponsored research in accordance with the DOE Public Access Plan https://www.energy.gov/downloads/doe-public-access-plan.

\bibliography{References_LPNets}
\appendix 
\section{Mathematical background on the dynamics of coupled bodies}
\label{app_sec_math_coupled}
\subsection{Notations for systems on Lie groups}\label{notations}

We quickly explain here the notations used 
throughout the paper for the treatment and derivation of the equations of motion for systems on Lie groups.


Let us denote by $ \operatorname{AD}_g:G \rightarrow G$, $ \operatorname{AD}_gh= gh g ^{-1}$ the action by conjugation. The adjoint action is $ \operatorname{Ad}_g: \mathfrak{g} \rightarrow \mathfrak{g} $, defined by $\operatorname{Ad}_g \xi = \left. \frac{d}{d\varepsilon}\right|_{\varepsilon=0} \operatorname{AD}_gc_ \varepsilon$ with $c_ \varepsilon \in G$ a curve tangent to $ \xi \in \mathfrak{g} $ at $e$ and $ \mathfrak{g} $ the Lie algebra of $G$. The Lie bracket is found as $[ \xi , \eta  ]= \operatorname{ad}_ \xi \eta =  \left. \frac{d}{d\varepsilon}\right|_{\varepsilon=0} \operatorname{Ad}_{c_ \varepsilon } \eta$, for $ \xi , \eta \in \mathfrak{g} $,
with $c_ \varepsilon \in G$ the same curve as before. Given a duality pairing $ \left\langle \cdot , \cdot \right\rangle : \mathfrak{g} ^* \times \mathfrak{g} \rightarrow \mathbb{R}  $ between $\mathfrak{g} $ and its dual $ \mathfrak{g} ^* $, the coadjoint action is $ \operatorname{Ad}^*_g: \mathfrak{g} ^* \rightarrow \mathfrak{g} ^*$, defined by $ \langle \operatorname{Ad}^*_g \mu , \xi \rangle = \langle \mu , \operatorname{Ad}_g \xi \rangle$, for all $ \xi \in \mathfrak{g} $ and $ \mu \in \mathfrak{g} ^*$. Then, the coadjoint operator appearing in the equations \eqref{variations_eqs_gen}, \eqref{eqs_two_groups}, \eqref{hamiltonian_eqs_two_groups} is defined by
\[
\operatorname{ad}^*_ \xi: \mathfrak{g} ^* \rightarrow \mathfrak{g} ^* , \quad \langle \operatorname{ad}^*_ \xi \mu , \eta \rangle= \langle \mu , \operatorname{ad}_ \xi \eta \rangle
\]
and satisfies $\operatorname{ad}^*_ \xi \mu= \left. \frac{d}{d\varepsilon}\right|_{\varepsilon=0} \operatorname{Ad}^*_{c_ \varepsilon } \mu$.

Given a Lagrangian $\ell$ as in \eqref{red_Lagr_2}, its partial derivatives $\frac{\partial \ell}{\partial \omega _i} \in \mathfrak{g} ^*$ and $ \frac{\partial \ell}{\partial p} \in T^*_pG$ are defined by
\[
\left\langle \frac{\partial \ell}{\partial \omega _1}, \delta \omega _1 \right\rangle = \left. \frac{d}{d\varepsilon}\right|_{\varepsilon=0} \ell( \omega _1+ \varepsilon \delta \omega _1, \omega _2,p)\,, \quad \left\langle \frac{\partial \ell}{\partial p}, \delta p \right\rangle = \left. \frac{d}{d\varepsilon}\right|_{\varepsilon=0} \ell( \omega _1, \omega _2, p_  \varepsilon )  \,,
\]
with $p_ \varepsilon \in G$ a curve with $ \left. \frac{d}{d\varepsilon}\right|_{\varepsilon=0} p_ \varepsilon = \delta p \in T_pG$ and where the duality pairing is between the tangent and cotangent spaces $T_pG$ and $T^*_pG$ at $p \in G$. The expression $p ^{-1} \frac{\partial \ell}{\partial p} \in \mathfrak{g} ^* $ appearing in equations \eqref{variations_eqs_gen} and \eqref{eqs_two_groups} is defined by
\begin{equation}\label{def_p} 
\left\langle p ^{-1} \frac{\partial \ell}{\partial p}, \xi \right\rangle := \left\langle \frac{\partial \ell}{\partial p}, p \xi \right\rangle \,  , \quad \text{for all $ \xi \in \mathfrak{g} $}\,,
\end{equation} 
with $p \xi \in T_pG$ defined by $p \xi := \left. \frac{d}{d\varepsilon}\right|_{\varepsilon=0} p c_ \varepsilon $ with $c_ \varepsilon \in G$ a curve with $c_0=e$ and $ \left. \frac{d}{d\varepsilon}\right|_{\varepsilon=0} c_ \varepsilon = \xi$. The partial derivatives $ \frac{\partial h}{\partial \mu _i} \in \mathfrak{g}$ and $ \frac{\partial h}{\partial p} \in T^*_pG$ appearing in \eqref{hamiltonian_eqs_two_groups} are defined similarly.

\subsection{Poisson brackets, Casimirs functions, and symplectic leaves}\label{Appendix_A2}

We recall that Poisson brackets on a manifod $P$ are bilinear skewsymmetric maps $\{ \cdot , \cdot \}: C^\infty(P) \times C^\infty(P) \rightarrow C^\infty(P)$ satisfying the Leibniz rule $\{fg,h\}= \{f,h\}g+ f\{g,h\}$ and the Jacobi identity $\{f,\{g,h\}\}+\{g,\{h,f\}\}+\{h,\{f,g\}\}=0$, for all $f,g,h \in C^\infty(P)$. The fact that the bracket \eqref{reduced_Poisson} for coupled systems defines a Poisson bracket on $P= \mathfrak{g} ^* \times \mathfrak{g} ^* \times G$ follows from 
Poisson reduction by symmetry of the canonical Poisson bracket on $T^*(G \times G)$. More precisely, the quotient map
\[
\pi : T^*( G \times G) \rightarrow T^*(G \times G)/G \simeq \mathfrak{g} ^* \times \mathfrak{g} ^* \times G\,,
\]
\[
\pi ( \alpha _{g_1}, \alpha  _{ g_2})= (g_1 ^{-1} \alpha _{g_1}, g_2 ^{-1} \alpha _{g_2}, g_1 ^{-1} g_2)= ( \mu _1, \mu _2, p)
\]
associated to the left action of $G$ on $T^*( G \times G)$, relates the bracket \eqref{reduced_Poisson} to the canonical Poisson bracket $\{ \cdot , \cdot \}_{\rm can}$ on $T^*(G \times G)$ as follows:
\[
\{ f \circ \pi , g \circ \pi \}_{\rm can}= \{f,g\} \circ \pi \,, \quad \text{for all $f,g$}.
\]
From this and general results in Poisson reduction, it immediately follows that \eqref{reduced_Poisson} is a Poisson bracket, see \cite{MaRa2013}.

Due to its non-canonical nature, this bracket admits \textit{Casimir functions}, i.e. functions $C: \mathfrak{g} ^* \times \mathfrak{g} ^* \times G \rightarrow \mathbb{R} $ which satisfy
\[
\{f,C\}=0\, , \quad  \text{for all functions $f$}.
\]

A class of Casimir functions is given in the following Lemma, in which we also consider the Lie-Poisson bracket on $ \mathfrak{g} ^* $ given by
\begin{equation}\label{LP_bracket}
\{F,G\}_{\rm LP}( \mu )= - \left\langle \mu , \left[ \frac{\partial F}{\partial \mu },  \frac{\partial G}{\partial \mu } \right] \right\rangle 
\end{equation} 
for $F,G \in C^\infty( \mathfrak{g} ^* )$.
Recall that this bracket also emerges by Poisson reduction, in the simpler situation of the quotient map
\[
\pi: T^*G \rightarrow T^*G/G\simeq \mathfrak{g} ^* \,, \quad \pi ( \alpha _g)= g ^{-1} \alpha _g =\mu 
\]
associated to the left action of $G$ on $T^*G$.

\begin{lemma}\label{Lemma_Casimir} Let $C_ \mathfrak{g} : \mathfrak{g} ^* \rightarrow \mathbb{R} $ be a Casimir function of the Lie-Poisson bracket on $ \mathfrak{g} ^* $, \emph{i.e.}, 
\[
\{F, C_ \mathfrak{g} \}_{\rm LP}=0\, , \quad  \text{for all functions $F: \mathfrak{g} ^* \rightarrow \mathbb{R}$}.
\]
Then the function $C: \mathfrak{g} ^* \times \mathfrak{g} ^* \times G \rightarrow \mathbb{R} $ defined by
\begin{equation}\label{Casimir_coupled} 
C( \mu _1, \mu _2, p)= C_ \mathfrak{g} ( \mu _1 + \operatorname{Ad}^*_ {p ^{-1} } \mu _2) 
\end{equation} 
is a Casimir function for the Poisson bracket \eqref{reduced_Poisson}.
\end{lemma}
\begin{proof} As a Casimir for \eqref{LP_bracket}, the function $C_ \mathfrak{g} ( \nu )$ satisfies
\begin{equation}\label{Casimir_prop} 
\operatorname{ad}^*_{ \frac{ \partial C_ \mathfrak{g}}{ \partial \nu  } } \nu =0.
\end{equation} 
The partial derivatives of $C$ are
\[
\frac{\partial C}{\partial \mu _1}= \frac{\partial C_ \mathfrak{g} }{\partial \nu}, \quad  \frac{\partial C}{\partial \mu _2}= \operatorname{Ad}_{p ^{-1} } \frac{\partial C_ \mathfrak{g} }{\partial \nu} , \quad  \frac{\partial C }{\partial p} p ^{-1}= \operatorname{ad}^*_{\frac{\partial C_ \mathfrak{g} }{\partial  \nu }} \operatorname{Ad}^*_{p ^{-1} } \mu  _2  = -\operatorname{ad}^*_{\frac{\partial C_ \mathfrak{g} }{\partial  \nu }} \mu  _1\,,
\]
where we used \eqref{Casimir_prop} in the last equality, as well as the formulas
\[
\delta (\operatorname{Ad}_{ p ^{-1} } \xi )= [ \operatorname{Ad}_{p ^{-1} } \xi , p ^{-1} \delta p ], \quad \delta (\operatorname{Ad}^*_{ p ^{-1} } \mu )
= - \operatorname{ad}^*_{ \delta p p ^{-1} }  \operatorname{Ad}^*_{p ^{-1} } \mu  
\]
for the variatons with respect to $p$. Inserting these expressions in \eqref{reduced_Poisson}, we get
\begin{align*}
\{f,C\}=& - \left\langle \mu _1, \left[ \frac{\partial f}{\partial \mu _1}, \frac{\partial C_ \mathfrak{g} }{\partial \nu }  \right] \right\rangle - \left\langle \mu _2, \left[ \frac{\partial f}{\partial \mu _2}, \operatorname{Ad}_{p ^{-1} } \frac{\partial C_ \mathfrak{g} }{\partial \nu }  \right] \right\rangle\\
&- \left\langle \frac{\partial f}{\partial \mu _1}, \operatorname{ad}^*_{ \frac{\partial C_ \mathfrak{g} }{\partial \nu } } \mu _1 \right\rangle  + \left\langle \frac{\partial f}{\partial \mu _2}, \operatorname{Ad}^*_p  \operatorname{ad}^*_{ \frac{\partial C_ \mathfrak{g} }{\partial \nu } } \mu _1 \right\rangle\\
=& - \left\langle \mu _1 + \operatorname{Ad}^*_{p ^{-1} } \mu _2, \left[ \operatorname{Ad}_p  \frac{\partial f}{\partial \mu _2}, \frac{\partial C_ \mathfrak{g} }{\partial \nu } \right] \right\rangle   =0\, ,
\end{align*}
where we used \eqref{Casimir_prop} in the last equality, with $ \nu = \mu _1+ \operatorname{Ad}^*_{p ^{-1} } \mu _2$.
\end{proof}

\medskip 

The class of Casimir functions given in \eqref{Casimir_coupled} is a particular case of the class considered in \citep{krishnaprasad1987hamiltonian} for coupled systems.

\medskip 

Finally, from a general result on Poisson structures, it follows that the phase space $ \mathfrak{g} ^* \times \mathfrak{g} ^* \times G$ of the Poisson bracket \eqref{reduced_Poisson} is foliated into symplectic leaves, on which the flow of the Hamiltonian system \eqref{hamiltonian_eqs_two_groups} restricts and becomes symplectic, \cite{MaRa2013}. More precisely, given a symplectic leaf $ \mathcal{L} \subset \mathfrak{g} ^* \times \mathfrak{g} ^* \times  G$, there exists a unique symplectic form $ \omega _ \mathcal{L} $ on $ \mathcal{L} $ such that the Poisson bracket, when restricted to $ \mathcal{L} $, becomes associated to the symplectic form $ \omega _ \mathcal{L} $:
\[
\{f,g\}|_ \mathcal{L} = \omega _ \mathcal{L} \big(X_{f|_ \mathcal{L} }, X_{g|_ \mathcal{L} }\big),
\]
where $X_{f|_ \mathcal{L} }$ is the Hamiltonian vector field associated to the function $f|_ \mathcal{L} : \mathcal{L} \rightarrow \mathbb{R} $ on the symplectic manifold $( \mathcal{L} , \omega _ \mathcal{L} )$, i.e. $ {\rm d} (f|_ \mathcal{L})= {\rm i} _{X_{f|_ \mathcal{L} }} \omega _ \mathcal{L} $. Note that Casimir functions are constant on symplectic leaves.

As we have shown, the CLPNets algorithms preserve the symplectic leaves and hence the Casimir functions.

\section{Derivation of the equations of motion for two coupled rigid bodies}
\label{app_derivation_SO3}

\paragraph{Notations and preliminaries} The Lie group to be considered in this case is $G=SO(3)$ with Lie algebra $ \mathfrak{so}(3)$ given by $3 \times 3$ antisymmetric matrices. 
We identify the dual $ \mathfrak{so}(3)^*$ with $\mathfrak{so}(3)$ by using the duality pairing $ \left\langle a , v \right\rangle = \frac{1}{2} \operatorname{Tr}(a^\mathsf{T} v)$, $a \in \mso(3)^*$, $v \in \mso(3)$.
We further identify $3 \times 3$ antisymmetric matrices $v$ with vectors $\mathbf{v}\in \mathbb{R} ^3$ by using the hat isomorphism $\;\;\widehat{  }\;: \mathbb{R} ^3  \rightarrow \mathfrak{so}(3)$, $ v= \widehat{ \mathbf{v}}$, defined by $ \mathbf{v} \times \mathbf{w} = v \mathbf{w}$, for all $ \mathbf{w} \in \mathbb{R} ^3  $. The inverse is written $\mathbf{v}= v^\vee$, satisfying $v^\vee \times \mathbf{w} = v \mathbf{w} $, for all $ \mathbf{w} \in \mathbb{R} ^3$. With this identification, we have
\begin{equation}\label{co_ad} 
\left\langle a,v \right\rangle = \mathbf{a} \cdot \mathbf{v}, \quad (\operatorname{ad}_ v u)^\vee= \mathbf{v} \times \mathbf{u} , \quad  \left( \operatorname{ad}^*_{v} a \right)^\vee = - \mathbf{v} \times \mathbf{a}\,.
\end{equation} 
The operator $\vee$ is extended from $ \mathfrak{so}(3)$ to arbitrary matrices $M$, by considering its skewsymmetric part as follows $M^\vee \times \mathbf{w}= \frac{1}{2} (M-M^\mathsf{T})\mathbf{w}$, for all $ \mathbf{w} \in \mathbb{R} ^3$. With this extension, we have the identity
\begin{equation}\label{identity}
(\mathbf{a} \otimes \mathbf{b})^\vee = \frac{1}{2} \left( \mathbf{a} \otimes \mathbf{b} - \mathbf{b} \otimes \mathbf{a} \right)^\vee  = - \frac{1}{2} (\mathbf{a} \times \mathbf{b})\,,
\end{equation} 
which will be used below.
The pairing between $v_p \in T_p SO(3)$ and $a_p \in T^*_pSO(3)$ is 
\begin{equation}\label{pairing_so3}
\langle a_p, v_p \rangle = \frac{1}{2}\operatorname{Tr}(a_p^\mathsf{T} v_p)=  \big(a_p p ^{-1} \big)^\vee \cdot \big(v_p p ^{-1} \big)^\vee
\end{equation}
and is left and right invariant. With this setting, by specifying the general equations \eqref{hamiltonian_eqs_two_groups} to the group or rotations, we get:
\begin{equation}
\left\{
\begin{array}{l} 
\displaystyle\vspace{0.2cm}\dot \bmu_1  =-\pp{h}{\bmu_1} \times \bmu_1 + \left( \frac{\partial h}{\partial p}p^\mathsf{T} \right)^\vee \\ 
\displaystyle\vspace{0.2cm}\dot \bmu_2  = -\pp{h}{\bmu_2} \times \bmu_2 - \left( p^\mathsf{T}\frac{\partial h}{\partial p} \right) ^\vee \\
\displaystyle\dot p  = - \pp{h}{\mu_1} p + p \pp{h}{\mu_2} \,.
\label{mu_1_mu2_F_gen_so3}
\end{array}\right. 
\end{equation}

\paragraph{Physical system and equations of motion} The physical setup of the problem is as follows. Let us consider two rigid bodies with their fixed point coinciding at $\mathbf{r}=0$. Each of the rigid bodies $1$ and $2$ has charges of value $q_{i,j}$, $j=1, \ldots,m$ with the coordinates $\bxi_{1,2}$ in each body frame. If the orientation of each rigid body is given by $\mathbb{R}_{1,2} \in SO(3)$, then the distance between the charges $i$ (body 1) and $j$ (body 2) are 
\begin{equation}
d_{ij}(p)=\left| \mathbb{R}_1 \bxi_{1,i} - \mathbb{R}_2 \bxi_{2,j} \right| = 
\left|  \bxi_{1,i} - p \bxi_{2,i} \right|\, , \quad p:=  \mathbb{R}_1^\mathsf{T}  \mathbb{R}_2 \, . 
    \label{d_ij}
\end{equation}
Assuming the tensors of inertia for each of the rigid body are $\mathbb{I}_i$, the physical Hamiltonian of this problem is $h: \mathfrak{so}(3) ^* \times \mathfrak{so}(3) ^* \times SO(3) \rightarrow \mathbb{R}$ given by
\begin{equation}
    h( \bmu_1,\bmu_2, p) = \frac{1}{2} \mathbb{I}_1^{-1}\bmu_1 \cdot \bmu_1 + \frac{1}{2} \mathbb{I}_2^{-2}\bmu_2 \cdot \bmu_2
    + \sum_{i<j} \frac{q_{1,i} q_{2,j}}{d_{ij}(p)}\,.
    \label{SO3_coupled}
\end{equation}
The computations of the derivatives proceed as follows: 
\begin{equation}
    \pp{h}{\bmu_i } = \mathbb{I}_i^{-1} \bmu_i \, , \quad 
    \pp{h}{p} = 2 \sum_{i<j} \frac{q_{1,i} q_{2,j}}{d_{ij}(p)^3} \left(  \bxi_{1,i} - p \bxi_{2,i} \right) \otimes \bxi_{2,i}\,,
    \label{h_derivs} 
\end{equation}
where we recall that we use the pairing \eqref{pairing_so3}. 
Then from \eqref{identity} and using  $(\bu \otimes \bv) p^\mathsf{T} =  \bu \otimes p \bv$, we get 
\begin{equation}
\begin{aligned} 
\left( \pp{h}{p} p^{-1} \right)^\vee & = -  \sum_{i<j} \frac{q_{1,i} q_{2,j}}{d_{ij}(p)^3}   \bxi_{1,i} \times p \bxi_{2,j}   
\\
\left( p^{-1} \pp{h}{p}  \right)^\vee  & =  -   \sum_{i<j} \frac{q_{1,i} q_{2,j}}{d_{ij}(p)^3}   p^\mathsf{T} \bxi_{1,i} \times  \bxi_{2,j} \,. 
\end{aligned}
    \label{dh_dp}
\end{equation}
Clearly, the force from charges acting on the first rigid body is the negative of the force acting on the second rigid body, multiplied by $p^\mathsf{T}$, as expected.  Substituting \eqref{h_derivs} and \eqref{dh_dp} into \eqref{mu_1_mu2_F_gen_so3}, we obtain the final equations of motion for two coupled rigid bodies \eqref{SO3_coupled_equations}. 

\section{Derivation of equations of motion for two coupled $SE(3)$ bodies}
\label{sec_derivation_SE3}

The coadjoint action is computed following the definition in \ref{notations}. For any two elements $g_1, g_2 \in SE(3)$, we have 
\begin{equation}
\operatorname{AD}_{g_1} g_2 = g_1 g_2 g_1 ^{-1} = 
\left( 
\begin{array}{cc} 
\mathbb{R}_1 \mathbb{R}_2 \mathbb{R}_1^\mathsf{T} & 
\mathbf{v}_1 + \mathbb{R}_1 \mathbf{v}_2 -
\mathbb{R}_1 \mathbb{R}_2 \mathbb{R}_1^\mathsf{T} 
\mathbf{v}_1 
\\
\mathbf{0}^\mathsf{T} & 1 
\end{array} 
\right) \, . 
    \label{AD_SE3}
\end{equation}
The $\operatorname{Ad}$-action is computed by considering a curve  $g_2 = g_2(t)$ at the identity and taking the derivative at $t=0$.  If we denote $\dot{g}_2(0) = (\widehat{\mathbf{a}_2},\mathbf{b}_2) \in \mse(3)$, then 
\begin{equation}
\begin{aligned} 
{\rm Ad}_{g_1=(\mathbb{R}_1, \mathbf{v}_1) } (\widehat{\mathbf{a}_2}, \mathbf{b}_2) & = 
\left( 
\begin{array}{cc} 
\mathbb{R}_1 \widehat{\mathbf{a}_2}\mathbb{R}_1^\mathsf{T} &  \mathbb{R}_1 \mathbf{b}_1 - \mathbb{R}_1 \widehat{\mathbf{a}_2} \mathbb{R}_1^\mathsf{T}\mathbf{v}_1 \\
    \mathbf{0}^\mathsf{T} & 0 
\end{array} 
\right)
\\
& := \left( \mathbb{R}_1 \widehat{\mathbf{a}_2} \mathbb{R}_1^\mathsf{T} ,  \mathbb{R}_1 \mathbf{b}_2 - \mathbb{R}_1 \widehat{\mathbf{a}_2} \mathbb{R}_1^\mathsf{T}\mathbf{v}_1\right) \, , 
    \end{aligned} 
    \label{Ad_SE3}
\end{equation}
where we have defined the element of the Lie algebra $\mse(3)$ by two components, collecting the meaningful non-zero terms in the $4 \times 4$ matrix. Note that since $\widehat{\mathbf{a}} \mathbf{v} = \mathbf{a} \times \mathbf{v}$, \eqref{Ad_SE3} can be written equivalently in vector form 
\begin{equation}
{\rm Ad}_{g_1=(\mathbb{R}_1, \mathbf{v}_1) } (\mathbf{a}_2,\mathbf{b}_2)  = 
\left( \mathbb{R}_1 \mathbf{a}_2  ,  \mathbb{R}_1 \mathbf{b}_2 - \mathbb{R}_1 \mathbf{a}_2  \times \mathbf{v}_1\right)\,.
    \label{Ad_SE3_vec}
\end{equation}

The ${\rm ad}$-action is then computed by additionally by considering $g_t(t)= (\mathbb{R}_1(t), \mathbf{v}_1(t)) $ and taking the derivative of that quantity at $t=0$. Again, with $\dot g_t(0)= (\widehat{a}_1(t), \mathbf{b}_1(t)) \in \mse(3)  $, we obtain, in the matrix and vector forms, respectively: 
\begin{equation}
\begin{aligned} 
{\rm ad}_{(\widehat{\mathbf{a}_1}, \mathbf{b}_1) } (\widehat{\mathbf{a}_2}, \mathbf{b}_2) & = 
\left( 
[ \widehat{\mathbf{a}_1}, \widehat{\mathbf{a}_2} ],  \widehat{\mathbf{a}_1} \mathbf{b}_2 - \widehat{\mathbf{a}_2} \mathbf{b}_1
\right)
\\
& := \left( 
\mathbf{a}_1 \times \mathbf{a}_2 ,  
\mathbf{a}_1 \times \mathbf{b}_2 - \mathbf{a}_2 \times  \mathbf{b}_1
\right) \,.
    \end{aligned} 
    \label{ad_SE3}
\end{equation}
Let us consider the pairing between an element $(\mathbf{a}, \mathbf{b}) \in \mse(3)$ and a dual element
$(\mathbf{A}, \mathbf{b}) \in \mse(3)^*$ defined by 
\begin{equation}
\left<  (\mathbf{a}, \mathbf{b}), 
(\mathbf{A}, \mathbf{B}) \right>  
= \frac{1}{2} \mbox{Tr} \big(\,\widehat{\mathbf{a}}\,^\mathsf{T}\, \widehat{\mathbf{A}} \,\big)  + \mathbf{b} \cdot \mathbf{B} = \mathbf{a} \cdot \mathbf{A} + \mathbf{b} \cdot \mathbf{B} \, . 
    \label{pairing_SE3}
\end{equation}
Then, with \eqref{pairing_SE3}, we can derive the expression for the coadjoint action ${\rm ad}^*$ as follows:
\begin{equation}
\begin{aligned}
 \big< {\rm ad}_{(\widehat{\mathbf{c}}, \mathbf{d}) } (\widehat{\mathbf{a}}, \mathbf{b}) \,   , 
(\widehat{ \mathbf{A}}, \mathbf{B})
\big>  &  = - \left( \mathbf{a} \times \mathbf{A} +\mathbf{b} \times \mathbf{B} \right) \cdot \mathbf{c} - 
\left(  \mathbf{a} \times \mathbf{B}\right) \cdot \mathbf{d} 
\\
& 
=\left< \left( 
-  \mathbf{a} \times \mathbf{A} -\mathbf{b} \times \mathbf{B}  \, , 
- \mathbf{a} \times \mathbf{B} \right) \, , 
( \mathbf{c}, \mathbf{d} ) 
\right> 
\\
& = 
\big< {\rm ad}^*_{(\mathbf{a},\mathbf{b})}(\mathbf{A}, \mathbf{B}), ( \mathbf{c}, \mathbf{d} ) \big> \,
\end{aligned} 
\end{equation} 
\begin{equation}\label{ad_star_SE3} \Rightarrow \quad {\rm ad}^*_{(\mathbf{a},\mathbf{b})}(\mathbf{A}, \mathbf{B})  = \left( - \left( \mathbf{a} \times \mathbf{A} +\mathbf{b} \times \mathbf{B} \right) \, , 
- \mathbf{a} \times \mathbf{B} \right)\,.
\end{equation}
We now compute the variations with respect to $p=g_1^{-1} g_2$ in order to derive the terms coming from the partial derivative $\pp{h}{p}$ in the equations of motion for $G=SE(3)$. Denote the variations as $\sigma _i = g_i^{-1} \de g_i = \big( \widehat{\boldsymbol{ \Sigma}}_{\boldsymbol{\alpha},i}, \boldsymbol{\Sigma}_{\boldsymbol{\beta},i} \big)$, where $ \widehat{\boldsymbol{\Sigma}}_{\boldsymbol{\alpha},i}$ and $\boldsymbol{\Sigma}_{\boldsymbol{\beta},i}$ are angular and linear momentum variations, respectively.  We proceed as follows: 
\begin{equation}
\begin{aligned} 
    \de \int_{t_0}^{t_f} h(p) \mbox{d} t & = 
    \int_{t_0}^{t_f} \left< \pp{h}{p} \, , \, \de p \right>_{SE(3)} {\rm d} t\\ 
    & = \int_{t_0}^{t_f} \left< \pp{h}{p}   \, , \, 
    \left( - \psi_1 p + p \psi_2  \right) \right>_{SE(3)} \mbox{d} t
    \\ 
    & = \int_{t_0}^{t_f} \left< \left( \pp{h}{\mathbb{Q}}, \pp{h}{\mathbf{v}} \right)  \, , \, 
    - \left( \widehat{\boldsymbol{\Sigma}}_{\boldsymbol{\alpha},1} \, , \,  \widehat{\boldsymbol{\Sigma}}_{\boldsymbol{\alpha},1} \mathbf{v} +  \boldsymbol{\Sigma}_{\boldsymbol{\beta},1} \right)  \right>_{SE(3)} \mbox{d} t
    \\ 
    & \qquad + \int_{t_0}^{t_f} \left< \left( \pp{h}{\mathbb{Q}}, \pp{h}{\mathbf{v}} \right)  \, , \, 
     \left( \mathbb{Q} \widehat{\boldsymbol{\Sigma}}_{\boldsymbol{\alpha},2} \, , \, \mathbb{Q} \boldsymbol{\Sigma}_{\boldsymbol{\beta},2}  \right)  \right>_{SE(3)} \mbox{d} t
         \\ 
    & = \int_{t_0}^{t_f} - \left(  \left( \pp{h}{\mathbb{Q}} \mathbb{Q}^\mathsf{T}\right)^\vee + \mathbf{v} \times \pp{h}{\mathbf{v}} \right) \cdot  
    \boldsymbol{\Sigma}_{\boldsymbol{\alpha},1}
    - 
    \pp{h}{\mathbf{v}} \cdot \boldsymbol{\Sigma}_{\boldsymbol{\beta},1} \mbox{d} t
    \\ 
    & \qquad +     \int_{t_0}^{t_f} \left(   \mathbb{Q}^\mathsf{T} \pp{h}{\mathbb{Q}} \right)^\vee  \cdot  
    \boldsymbol{\Sigma}_{\boldsymbol{\alpha},2}
    + 
    \mathbb{Q}^\mathsf{T} \pp{h}{\mathbf{v}} \cdot \boldsymbol{\Sigma}_{\boldsymbol{\beta},2} \mbox{d} t\,.
    \end{aligned} 
    \label{delta_h_var}
\end{equation}
The terms proportional to 
$ \boldsymbol{\Sigma}_{\boldsymbol{\alpha},i}$ 
yield the terms in the equation for $\dot{\boldsymbol{\alpha}}_i$, while those proportional to 
$ \boldsymbol{\Sigma}_{\boldsymbol{\beta},i}$ 
contribute to the equation for $\dot{\boldsymbol{\beta}}_i$. 

By using the expressions \eqref{ad_star_SE3} and \eqref{delta_h_var} in the general equations \eqref{hamiltonian_eqs_two_groups}, we obtain the equations given by \eqref{SE3_equations_explicit} for the angular and linear momenta 
$(\boldsymbol{\alpha}_i, \boldsymbol{\beta}_i) \in \mse(3) $ and relative orientation and translation $p = (\mathbb{Q}, \mathbf{v}) \in SE(3)$. 
\end{document}